\documentclass{article}

\usepackage{xr}

\usepackage[utf8]{inputenc} 
\usepackage[T1]{fontenc}    
\usepackage{booktabs}       
\usepackage{amsfonts}       
\usepackage{nicefrac}       
\usepackage{microtype}      

\usepackage[usenames,dvipsnames]{color}

\usepackage{amsmath,bbm,bm,graphicx,amsthm}
\usepackage{subfigure}

\newtheorem{theorem}{Theorem}
\newtheorem{lemma}{Lemma}
\newtheorem{assumption}{Assumption}

\DeclareMathOperator*{\KL}{KL}

\DeclareMathOperator{\Tr}{Tr}
\DeclareMathOperator{\Var}{Var}
\DeclareMathOperator{\Sup}{Sup}
\DeclareMathOperator{\Gen}{Gen}

\usepackage{hyperref}

\usepackage[accepted]{icml2021}

\icmltitlerunning{Positive-Negative Momentum}

\begin{document}

\twocolumn[
\icmltitle{Positive-Negative Momentum: \\ Manipulating Stochastic Gradient Noise to Improve Generalization}

\begin{icmlauthorlist}
\icmlauthor{Zeke Xie}{utokyo,riken}
\icmlauthor{Li Yuan}{nus}
\icmlauthor{Zhanxing Zhu}{bibdr}
\icmlauthor{Masashi Sugiyama}{riken,utokyo}
\end{icmlauthorlist}

\icmlaffiliation{utokyo}{The University of Tokyo}
\icmlaffiliation{riken}{RIKEN Center for AIP}
\icmlaffiliation{nus}{National University of Singapore}
\icmlaffiliation{bibdr}{Beijing Institute of Big Data Research}

\icmlcorrespondingauthor{Zeke Xie}{xie@ms.k.u-tokyo.ac.jp}

\icmlkeywords{Deep Learning, Gradient Noise, Momentum, Optimization}

\vskip 0.3in
]



\printAffiliationsAndNotice{}  

\begin{abstract}
It is well-known that stochastic gradient noise (SGN) acts as implicit regularization for deep learning and is essentially important for both optimization and generalization of deep networks. Some works attempted to artificially simulate SGN by injecting random noise to improve deep learning. However, it turned out that the injected simple random noise cannot work as well as SGN, which is anisotropic and parameter-dependent. For simulating SGN at low computational costs and without changing the learning rate or batch size, we propose the Positive-Negative Momentum (PNM) approach that is a powerful alternative to conventional Momentum in classic optimizers. The introduced PNM method maintains two approximate independent momentum terms. Then, we can control the magnitude of SGN explicitly by adjusting the momentum difference. We theoretically prove the convergence guarantee and the generalization advantage of PNM over Stochastic Gradient Descent (SGD). By incorporating PNM into the two conventional optimizers, SGD with Momentum and Adam, our extensive experiments empirically verified the significant advantage of the PNM-based variants over the corresponding conventional Momentum-based optimizers.
\end{abstract}

\section{Introduction}
\label{sec:intro}

Stochastic optimization methods, such as Stochastic Gradient Descent (SGD), have been popular and even necessary in training deep neural networks \citep{lecun2015deep}. It is well-known that stochastic gradient noise (SGN) in stochastic optimization acts as implicit regularization for deep learning and is essentially important for both optimization and generalization of deep neural networks \citep{hochreiter1995simplifying,hochreiter1997flat,hardt2016train,xie2020diffusion,wu2020direction}. The theoretical mechanism behind the success of stochastic optimization, particularly the SGN, is still a fundamental issue in deep learning theory. 

\textbf{SGN matters.} From the viewpoint of minima selection, SGN can help find flatter minima, which tend to yield better generalization performance \citep{hochreiter1995simplifying,hochreiter1997flat}. For this reason, how SGN selects flat minima has been investigated thoroughly. For example, \citet{zhu2019anisotropic} quantitatively revealed that SGN is better at selecting flat minima than trivial random noise, as SGN is anisotropic and parameter-dependent. \citet{xie2020diffusion} recently proved that, due to the anisotropic and parameter-dependent SGN, SGD favors flat minima even exponentially more than sharp minima. From the viewpoint of optimization dynamics, SGN can accelerate escaping from saddle points via stochastic perturbations to gradients~\citep{jin2017escape,daneshmand2018escaping,staib2019escaping,xie2020adai}.

\textbf{Manipulating SGN.} Due to the benefits of SGN, improving deep learning by manipulating SGN has become a popular topic. There are mainly two approaches along this line. 

The first approach is to control SGN by tuning the hyperparameters, such as the learning rate and batch size, as the magnitude of SGN has been better understood recently. It is well-known \citep{mandt2017stochastic} that the magnitude of SGN in continuous-time dynamics of SGD is proportional to the ratio of the learning rate $\eta$ and the batch size $B$, namely $\frac{\eta}{B}$. \citet{he2019control} and \citet{li2019towards} reported that increasing the ratio $\frac{\eta}{B}$ indeed can improve test performance due to the stronger implicit regularization of SGN. However, this method is limited and not practical for at least three reasons. First, training with a too small batch size is computationally expensive per epoch and often requires more epochs for convergence \citep{hoffer2017train,zhang2019algorithmic}. Second, increasing the learning rate only works in a narrow range, since too large initial learning rates may lead to optimization divergence or bad convergence \citep{keskar2017large,masters2018revisiting,xie2020stable}. In practice, one definitely cannot use a too large learning rate to guarantee good generalization \citep{keskar2017large,masters2018revisiting}. Third, decaying the ratio $\frac{\eta}{B}$ during training (via learning rate decay) is almost necessary for the convergence of stochastic optimization as well as the training of deep networks \citep{smith2018don}. Thus, controlling SGN by adjusting the ratio $\frac{\eta}{B}$ cannot consistently be performed during the entire training process.

The second approach is to simulate SGN by using artificial noise. Obviously, we could simulate SGN well only if we clearly understood the structure of SGN. Related works \citep{daneshmand2018escaping,zhu2019anisotropic,xie2020diffusion,wen2020empirical} studied the noise covariance structure, and there still exists a dispute about the noise type of SGD and its role in generalization \citep{simsekli2019tail,xie2020diffusion}. It turned out that, while injecting small Gaussian noise into SGD may improve generalization  \citep{an1996effects,neelakantan2015adding,zhou2019toward,xie2020artificial}, unfortunately, Gradient Descent (GD) with artificial noise still performs much worse than SGD (i.e., SGD can be regarded as GD with SGN). \citet{wu2020noisy} argued that GD with multiplicative sampling noise may generalize as well as SGD, but the considered SGD baseline was weak due to the absence of weight decay and common training tricks. Thus, this approach cannot work well in practice.

 \begin{table*}[t]
\caption{PNM versus conventional Momentum. We report the mean and the standard deviations (as the subscripts) of the optimal test errors computed over three runs of each experiment. The proposed PNM-based methods show significantly better generalization than conventional momentum-based methods. Particularly, as Theorem \ref{pr:pnmgeneralization} indicates, Stochastic PNM indeed consistently outperforms the conventional baseline, SGD.}
\label{table:cifar}
\begin{center}
\begin{small}
\begin{sc}
\resizebox{\textwidth}{!}{%
\begin{tabular}{ll | lllllllllllll}
\toprule
Dataset & Model  & PNM & AdaPNM & SGD &Adam  &AMSGrad & AdamW & AdaBound  & Padam & Yogi & RAdam  \\
\midrule
CIFAR-10   &ResNet18        & $\mathbf{4.48}_{0.09}$  & $4.94_{0.05}$  & $5.01_{0.03}$  & $6.53_{0.03} $ & $6.16_{0.18}$ &  $5.08_{0.07}$ &    $5.65_{0.08}$ & $5.12_{0.04}$ & $5.87_{0.12}$ & $6.01_{0.10}$   \\
                   &VGG16            & $6.26_{0.05}$  &  $\mathbf{5.99}_{0.11}$  & $6.42_{0.02}$   & $7.31_{0.25} $  & $7.14_{0.14}$  & $6.48_{0.13}$ &  $6.76_{0.12}$ & $6.15_{0.06}$ & $6.90_{0.22}$ & $6.56_{0.04}$  \\
CIFAR-100  &ResNet34    & $20.59_{0.29}$ &  $\mathbf{20.41}_{0.18}$ & $21.52_{0.37}$  & $27.16_{0.55} $  & $25.53_{0.19}$ &  $22.99_{0.40}$ & $22.87_{0.13}$ & $22.72_{0.10}$ & $23.57_{0.12}$ & $24.41_{0.40}$   \\
                  &DenseNet121   &  $\mathbf{19.76}_{0.28}$  & $20.68_{0.11}$ & $19.81_{0.33}$  & $25.11_{0.15} $  & $24.43_{0.09}$ &  $21.55_{0.14}$ &  $22.69_{0.15}$ & $21.10_{0.23}$ & $22.15_{0.36}$ & $22.27_{0.22}$  \\
                   &GoogLeNet     & $20.38_{0.31}$ &  $\mathbf{20.26}_{0.21}$ & $21.21_{0.29}$  & $26.12_{0.33} $  & $25.53_{0.17}$ &  $21.29_{0.17}$ &  $23.18_{0.31}$ & $21.82_{0.17}$ & $24.24_{0.16}$ & $22.23_{0.15}$ \\
                    
\bottomrule
\end{tabular}
}
\end{sc}
\end{small}
\end{center}
\end{table*}

\textbf{Contribution.} Is it possible to manipulate SGN without changing the learning rate or batch size? Yes. In this work, we propose Positive-Negative Momentum\footnote{Code: \url{https://github.com/zeke-xie/Positive-Negative-Momentum}.}  (PNM) for enhancing SGN at the low computational and coding costs. We summarize four contributions in the following.
\begin{itemize}
\item We introduce a novel method for manipulating SGN without changing the learning rate or batch size. The proposed PNM strategy can easily replace conventional Momentum in classical optimizers, including SGD and Adaptive Momentum Estimation (Adam).
\item We theoretically prove that PNM has a convergence guarantee similar to conventional Momentum. 
\item Within the PAC-Bayesian framework, we theoretically prove that PNM may have a tighter generalization bound than SGD.
\item We provide extensive experimental results to verify that PNM can indeed make significant improvements over conventional Momentum, shown in Table \ref{table:cifar}.
\end{itemize}

In Section \ref{sec:method}, we introduce the proposed methods and the motivation of enhancing SGN. In Section \ref{sec:convergence}, we present the convergence analysis. In Section \ref{sec:generalization}, we present the generalization analysis. In Section \ref{sec:empirical}, we conduct empirical analysis. In Section \ref{sec:conclusion}, we conclude our main work.

\section{Methodology}
\label{sec:method}

In this section, we introduce the proposed PNM method and explain how it can manipulate SGN. 

\textbf{Notation.} Suppose the loss function is $f(\theta)$, $\theta$ denotes the model parameters, the learning rate is $\eta$, the batch size is $B$, and the training data size is $N$. The basic gradient-descent-based updating rule can written as
\begin{align}
\theta_{t+1} = \theta_{t} - \eta g_{t}.
\end{align}
Note that $g_{t} = \nabla f(\theta_{t})$ for deterministic optimization, while $g_{t} = \nabla f(\theta_{t}) + \xi_{t} $ for stochastic optimization, where $ \xi_{t} $ indicates SGN. As SGN is from the difference between SGD and GD and the minibatch samples are uniformly chosen from the whole training dataset, it is commonly believed that $g_{t}$ is an unbiased estimator of the true gradient $\nabla f(\theta_{t})$ for stochastic optimization. Without loss of generality, we only consider one-dimensional case in Sections \ref{sec:method} and \ref{sec:convergence}.The mean and the variance of SGN can be rewritten as $\mathbb{E}[\xi] = 0$ and $\Var(\xi) = \sigma^{2}$, respectively. We may use $\sigma$ as a measure of the noise magnitude.

\textbf{Conventional Momentum.} 
We first introduce the conventional Momentum method, also called Heavy Ball (HB), seen in Algorithm \ref{algo:shb} \citep{zavriev1993heavy}. We obtain vanilla SGD by $\beta_{1}=0$ and $\beta_{3}=1$, and obtain common SGD with Momentum by $\beta_{1}=0.9$ and $\beta_{3}=1$. Algorithm \ref{algo:shb} is the actual PyTorch SGD\citep{paszke2019pytorch}, and can be theoretically reduced to SGD with a different learning rate. Adam uses the exponential moving average of past stochastic gradients as momentum by $\beta_{3}=1-\beta_{1}$. The conventional Momentum can be written as
\begin{align}
m_{t} = \sum_{k=0}^{t} \beta_{3} \beta_{1}^{t-k} g_{k},
\end{align}
which is the estimated gradient for updating model parameters. Then we approximately have $\mathbb{E}[m] \approx \frac{\beta_{3}}{1-\beta_{1}} \nabla f(\theta)$. The stochastic noise in momentum is given by
\begin{align}
\label{eq:hbnoise}
\xi^{\mathrm{hb}}_{t} = \sum_{k=0}^{t} \beta_{3} \beta_{1}^{t-k} \xi_{k}.
\end{align}
 Without loss of generality, we use the Adam-style Momentum with $\beta_{3}=1-\beta_{1}$ in our analysis. Thus, the conventional momentum does not control the gradient noise magnitude, because, for large $t$, 
\begin{align}
\label{eq:hbnoisevar}
\Var(\xi^{\mathrm{hb}}) = \beta_{3} \frac{1-\beta_{1}^{t+1}}{1-\beta_{1}} \sigma^{2} \approx  \sigma^{2} .
\end{align}
We have assumed that SGN $\xi$ is approximately independent. Since this approximation holds well in the limit of $\frac{B}{N} \rightarrow 0$, this assumption is common in theoretical analysis \citep{mandt2017stochastic}.

\textbf{Basic Idea.} Our basic idea for manipulating SGN is simple. Suppose that $g^{(a)}$ and $g^{(b)}$ are two independent unbiased estimators of $\nabla f(\theta)$. Then their weighted average is
\begin{align}
\bar{g} = & (1+\beta_{0}) g^{(a)} - \beta_{0} g^{(b)} \nonumber \\
        = &  \nabla f(\theta) + \bar{\xi},
\end{align}
where $\bar{\xi} = (1+\beta_{0}) \xi^{(a)} - \beta_{0} \xi^{(b)}$. If $\beta_{0} > 0$, for the generated noisy gradient $\bar{g}$, we have $\mathbb{E}[\bar{g}] = \nabla f(\theta) $ and $\Var(\bar{\xi}) = [(1+\beta_{0})^{2} + \beta_{0}^{2}] \sigma^{2} $. In this way, we can control the noise magnitude by $\beta_{0}$ without changing the expectation of the noisy gradient.

\textbf{Positive-Negative Momentum.} 
Inspired by this simple idea, we combine the positive-negative averaging with the conventional Momentum method. For simplicity, we assume that $t$ is an odd number. We maintain two independent momentum terms as
\begin{align}
\begin{cases}
& m_{t}^{\mathrm{(odd)}} = \sum_{k=1,3,\ldots,t} \beta_{3} \beta_{1}^{t-k} g_{k} , \\
& m_{t}^{\mathrm{(even)}} = \sum_{k=0,2,\ldots,t-1} \beta_{3} \beta_{1}^{t-k} g_{k} ,
 \end{cases}  
\end{align}
by using two alternate sequences of past gradients, respectively. Then we use the positive-negative average,
\begin{align}
\label{eq:pnm}
m_{t} = (1+\beta_{0}) m_{t}^{\mathrm{(odd)}} - \beta_{0} m_{t}^{\mathrm{(even)}},
\end{align}
as an estimated gradient for updating model parameters. Similarly, if $t$ is an even number, we let $m_{t} = (1+\beta_{0}) m_{t}^{\mathrm{(even)}} - \beta_{0} m_{t}^{\mathrm{(odd)}}$. When $\beta_{0} > 0$, one momentum term has a positive coefficient and another one has a negative coefficient. Thus, we call it a positive-negative momentum pair. 

The stochastic noise $\xi^{\mathrm{pnm}}_{t} $ in the positive-negative momentum pair is given by
\begin{align}
\label{eq:pnmnoise}
   \xi^{\mathrm{pnm}}_{t}  = & (1+\beta_{0})  \sum_{k=1,3,\ldots,t} \beta_{3} \beta_{1}^{t-k} \xi_{k}  -  \nonumber \\
 &   \beta_{0} \sum_{k=0,2,\ldots,t-1} \beta_{3} \beta_{1}^{t-k} \xi_{k}.
\end{align}
For large $t$, we write the noise variance as 
\begin{align}
\label{eq:pnmnoisevar}
\Var(\xi^{\mathrm{pnm}}) \approx [(1+\beta_{0})^{2} + \beta_{0}^{2}] \sigma^{2} .
\end{align}
The noise magnitude of positive-negative momentum in Equation \eqref{eq:pnm} is $\sqrt{(1+\beta_{0})^{2} + \beta_{0}^{2}}$ times the noise magnitude of conventional Momentum. 

While computing the gradients twice per iteration by using two minibatches can be also used for constructing positive-negative pairs, using past stochastic gradients has two benefits: lower computational costs and lower coding costs. First, avoiding computing the gradients twice per iteration save computation costs. Second, we may implement the proposed method inside optimizers, which can be employed in practice more easily.

\begin{algorithm}
   \caption{(Stochastic) Heavy Ball/Momentum}
   \label{algo:shb}
      $ m_{t} = \beta_{1} m_{t-1} + \beta_{3} g_{t}$\;  \\
      $ \theta_{t+1} = \theta_{t} -  \eta m_{t} $\;
\end{algorithm}

\begin{algorithm}
 \caption{(Stochastic) PNM}
 \label{algo:spnm}
      $ m_{t} = \beta_{1}^{2} m_{t-2} + (1-\beta_{1}^{2}) g_{t} $\; \\
      $ \theta_{t+1} = \theta_{t} - \frac{ \eta }{\sqrt{(1 +  \beta_{0} )^{2} + \beta_{0}^{2}}}  [ (1 +  \beta_{0} ) m_{t}  - \beta_{0} m_{t-1} ] $\; %
\end{algorithm}

\begin{algorithm}
 \caption{AdaPNM} 
 \label{algo:adapnm}
      $ m_{t} = \beta_{1}^{2} m_{t-2} + (1-\beta_{1}^{2}) g_{t} $\; \\
      $ \hat{m}_{t} =\frac{ (1 +  \beta_{0} ) m_{t}  - \beta_{0} m_{t-1} }{(1-\beta_{1}^{t})}   $\; \\
      $ v_{t} = \beta_{2}  v_{t-1} + (1 - \beta_{2} ) g_{t}^{2} $\; \\
      $v_{\mathrm{max}} = \max(v_{t}, v_{\mathrm{max}})$\;  \\
      $ \hat{v}_{t} = \frac{v_{\mathrm{max}}}{1-\beta_{2}^{t}} $\;  \\
      $ \theta_{t+1} = \theta_{t} -  \frac{\eta}{\sqrt{(1 +  \beta_{0} )^{2} + \beta_{0}^{2}}(\sqrt{\hat{v}_{t}} + \epsilon)} \hat{m}_{t} $\; 
\end{algorithm}

\textbf{Algorithms.}
We further incorporate PNM into SGD and Adam, and propose two novel PNM-based variants, including (Stochastic) PNM in Algorithm \ref{algo:spnm} and AdaPNM in Algorithm \ref{algo:adapnm}. Note that, by letting $\beta_{0} = - \frac{\beta_{1}}{1+\beta_{1}}$, we may recover conventional Momentum and Adam as the special cases of Algorithm \ref{algo:spnm} and Algorithm \ref{algo:adapnm}. Note that our paper uses the AMSGrad variant in Algorithm \ref{algo:adapnm} unless we specify it. Because \citet{reddi2019convergence} revealed that the AMSGrad variant to secure the convergence guarantee of adaptive gradient methods. We supplied AdaPNM without AMSGrad in Appendix.

\textbf{Normalizing the learning rate.} 
We particularly normalize the learning rate by the noise magnitude as $\frac{ \eta }{\sqrt{(1 +  \beta_{0} )^{2} + \beta_{0}^{2}}}$ in the proposed algorithms. The ratio of the noise magnitude to the learning rate matters to the convergence error (see Theorem \ref{pr:pnmconverge} below). In practice (not the long-time limit), it is important to achieve low convergence errors in the same epochs as SGD. Practically, we also observe in experiments that using the learning rate as $\frac{ \eta }{\sqrt{(1 +  \beta_{0} )^{2} + \beta_{0}^{2}}}$ for PNM can free us from re-tuning the hyperparameters, while we will need to re-fine-tune the learning rate without normalization, which is time-consuming in practice. Note that PNM can have a much larger ratio of the noise magnitude to learning rate than SGD.

\section{Convergence Analysis}
\label{sec:convergence}

In this section, we theoretically prove the convergence guarantee of Stochastic PNM. 

By Algorithm \ref{algo:spnm}, we may rewrite the main updating rules as 
\begin{align*}
\begin{cases}
      & m_{t} = \beta_{1}^{2} m_{t-2} + (1 - \beta_{1}^{2}) g_{t} , \\
      &(\theta_{t+1}  + \eta_{0}\beta_{0} m_{t}) = (\theta_{t} + \eta_{0}\beta_{0} m_{t-1}) - \eta_{0} m_{t},
 \end{cases}  
\end{align*}
where $ \eta_{0} =  \frac{ \eta }{\sqrt{(1 +  \beta_{0} )^{2} + \beta_{0}^{2}}}$. We denote that $\theta_{-2} = \theta_{-1} = \theta_{0}$, $\beta = \beta_{1}^{2}$, and $\alpha = \eta  \frac{ (1-\beta) }{\sqrt{(1 +  \beta_{0} )^{2} + \beta_{0}^{2}}} $. Then PNM can be written as
\begin{align}
\begin{cases}
      &x_{t} = \theta_{t} + \eta_{0} \beta_{0} m_{t-1}, \\
      &x_{t+1} = x_{t} - \alpha g_{t} + \beta (x_{t-1} - x_{t-2}).
 \end{cases}  
\end{align}
Note that $x_{t+1} - x_{t} = \alpha m_{t}$. We may also write SGD with Momentum as 
\begin{align}
      \theta_{t+1} = \theta_{t} - \eta g_{t} + \beta (\theta_{t} - \theta_{t-1}).
\end{align}
Obviously, PNM maintains two approximately independent momentum terms by using past odd-number-step gradients and even-number-step gradients, respectively. 

Inspired by \citet{yan2018unified}, we propose Theorem \ref{pr:pnmconverge} and prove that Stochastic PNM has a similar convergence rate to SGD with Momentum. The errors given by Stochastic PNM and SGD with Momentum are both $\mathcal{O}(\frac{1}{\sqrt{t}})$ \citep{yan2018unified}. We leave all proofs in Appendix \ref{sec:appproofs}.

 \begin{theorem}[Convergence of Stochastic Positive-Negative Momentum]
 \label{pr:pnmconverge}
Assume that $f(\theta)$ is a $L$-smooth function, $f$ is lower bounded as $f(\theta) \geq f^{\star} $, $\mathbb{E}[\xi ] = 0$, $\mathbb{E}[\| g(\theta, \xi) - \nabla f(\theta) \|^{2}] \leq \sigma^{2} $, and $\| \nabla f(\theta)  \| \leq G$ for any $\theta$. Let $\beta_{1} \in [0, 1)$, $\beta_{0} \geq 0$ and Stochastic PNM run for $t+1$ iterations. If $ \frac{\eta}{ \sqrt{(1+\beta_{0})^{2} + \beta_{0}^{2}}} = \min \{ \frac{1}{2L}, \frac{C}{\sqrt{t+1}} \}$, we have
\begin{align*}
 & \min_{k=0,\ldots, t} \mathbb{E}[ \| \nabla f(\theta_{k})\|^{2}]   \\
 \leq & \frac{ 2(f(\theta_{0}) -f^{\star})}{t+1} \max \{ 2L, \frac{\sqrt{t+1}}{C} \} +  \frac{C_{1}}{\sqrt{t+1}} ,
\end{align*}
where
\begin{align*}
C_{1} = C \frac{ L (\beta + \beta_{0} (1-\beta))^{2} (G^{2} + \sigma^{2}) + L(1-\beta)^{2} \sigma^{2}}{(1-\beta)^{2}}.
\end{align*}
\end{theorem}

\section{Generalization Analysis}
\label{sec:generalization}

In this section, we theoretically prove the generalization advantage of Stochastic PNM over SGD by using PAC-Bayesian framework \citep{mcallester1999some,mcallester1999pac}. We discover that, due to the stronger SGN, the posterior given by Stochastic PNM has a tighter generalization bound than the SGD posterior.

\subsection{The posterior analysis}

\citet{mcallester1999some} pointed that the posterior given by a training algorithm is closely related to the generalization bound. Thus, we first analyze the posteriors given by SGD, Momentum, and Stochastic Positive-Negative Momentum. 
\citet{mandt2017stochastic} studied the posterior given by the continuous-time dynamics of SGD. We present Assumptions \ref{as:taylor} in \citet{mandt2017stochastic}, which is true near minima. Note that $H(\theta)$ denotes the Hessian of the loss function $f$ at $\theta$.

\begin{assumption}[The second-order Taylor approximation]
 \label{as:taylor}
 The loss function around a minimum $\theta^{\star}$ can be approximately written as 
 \[ f(\theta) = f(\theta^{\star}) + \frac{1}{2}(\theta -\theta^{\star})^{\top} H(\theta^{\star}) (\theta -\theta^{\star}).\]
\end{assumption}

When the noisy gradient for each iteration is the unbiased estimator of the true gradient, the dynamics of gradient-based optimization can be always written as
\begin{align}
\label{eq:sgddiscrete}
\theta_{t+1} = \theta_{t} - \eta (\nabla f(\theta) + C(\theta) \xi),
\end{align}
where $C(\theta)$ is the covariance of gradient noise and $\xi$ obeys the standard Gaussian distribution $\mathcal{N}(0,I)$. While recent papers \citep{simsekli2019tail,xie2020diffusion} argued about the noise types, we follow most papers \citep{welling2011bayesian,jastrzkebski2017three,li2017stochastic,mandt2017stochastic,hu2019diffusion,xie2020diffusion} in this line and still approximate SGN as Gaussian noise due to Central Limit Theorem. The corresponding continuous-time dynamics can be written as 
\begin{align}
\label{eq:sgdcontinuous}
 d \theta =  - \nabla f(\theta) dt + [\eta C(\theta)]^{\frac{1}{2}} dW_{t},
 \end{align}
 where $dW_{t} = \mathcal{N}(0,I dt)$ is a Wiener process.

\citet{mandt2017stochastic} used SGD as an example and demonstrated that, the posterior generated by Equation \eqref{eq:sgdcontinuous} in the local region around a minimum $\theta^{\star}$ is a Gaussian distribution $\mathcal{N}(\theta^{\star},\Sigma_{sgd})$. This well-known result can be formulated as Theorem \ref{pr:sgdposterior}. We denote that $H(\theta^{\star}) = H$ and $C(\theta^{\star}) = C$ in the following analysis.

 \begin{theorem}[The posterior generated by Equation \eqref{eq:sgdcontinuous} near a minimum \citep{mandt2017stochastic}]
 \label{pr:sgdposterior}
Suppose that Assumption \ref{as:taylor} holds, the covariance near the minimum $\theta^{\star}$ is $C(\theta^{\star})$, the dynamics is governed by Equation \eqref{eq:sgdcontinuous}. Then, in the long-time limit, the generated posterior $Q$ near $\theta^{\star}$ is a Gaussian distribution $\mathcal{N}(\theta^{\star},\Sigma)$. and the $\Sigma$ satisfies 
 \[ \Sigma H(\theta^{\star}) + H(\theta^{\star}) \Sigma = \eta C(\theta^{\star}) . \]
\end{theorem}

\textbf{The vanilla SGD posterior.} In the case of SGD, based on \citet{jastrzkebski2017three,zhu2019anisotropic,xie2020diffusion}, the covariance $C(\theta)$ is proportional to the Hessian $H(\theta)$ and inverse to the batch size $B$ near minima: 
\begin{align}
\label{eq:covarhessian}
C_{sgd}(\theta) \approx  \frac{1}{B} H(\theta).
\end{align}
Equation \eqref{eq:covarhessian} has been theoretically and empirically studied by related papers \citep{xie2020diffusion,xie2020adai}. Please see Appendix \ref{sec:noiseanalysis} for the details.

By Equation \eqref{eq:covarhessian} and Theorem \ref{pr:sgdposterior}, we may further express $\Sigma_{sgd}$ as 
\begin{align}
\label{eq:sgdsigma}
\Sigma_{sgd} = \frac{\eta}{2B} I,
\end{align}
where $I$ is the $n \times n$ identity matrix and $n$ is the number of model parameters.

\textbf{The Momentum posterior.} By Equation \eqref{eq:hbnoise}, we know that, in continuous-time dynamics, we may write the noise covariance in HB/SGD with Momentum as
\begin{align}
C_{\mathrm{hb}}(\theta) = \frac{\beta_{3}}{1-\beta_{1}} C_{sgd}  = C_{sgd}.
\end{align}
Without loss of generality, we have assumed that $\beta_{3} = 1- \beta_{1}$ in HB/SGD with Momentum. Then, in the long-time limit, we further express $\Sigma_{\mathrm{hb}}$ as 
\begin{align}
\Sigma_{\mathrm{hb}} = \frac{\beta_{3}}{1-\beta_{1}} \Sigma_{sgd}  = \frac{\eta}{2B} I.
\end{align}
This result has also been obtained by \citet{mandt2017stochastic}. Thus, the Momentum posterior is approximately equivalent to the vanilla SGD posterior. In the PAC-Bayesian framework (Theorem \ref{pr:pac}), HB should generalize as well as SGD in the long-time limit.

\textbf{The Stochastic PNM posterior.}  Similarly, by Equation \eqref{eq:pnmnoise}, we know that, in continuous-time dynamics, we may write the noise covariance in Stochastic PNM as
\begin{align}
C_{\mathrm{pnm}}(\theta) = [(1+\beta_{0})^{2} + \beta_{0}^{2}] C_{sgd}.
\end{align}
where $C_{\mathrm{pnm}}(\theta) = ((1+\beta_{0})^{2} + \beta_{0}^{2}) C_{sgd}(\theta)$ is the covariance of SGN in Stochastic PNM.

By Equation \eqref{eq:covarhessian} and Theorem \ref{pr:sgdposterior}, in the long-time limit, we may further express $\Sigma_{\mathrm{pnm}}$ as 
\begin{align}
\label{eq:pnmsigma}
\Sigma_{\mathrm{pnm}} =  [(1+\beta_{0})^{2} + \beta_{0}^{2}]  \frac{\eta}{2B} I.
\end{align}
Thus, we may use the hyperparameter $\beta_{0}$ to rescale the covariance of the posterior. In the following analysis, we will prove that the PAC-Bayesian bound may heavily depend on the new hyperparameter $\beta_{0}$. 

\subsection{The PAC-Bayesian bound analysis}

\textbf{The PAC-Bayesian generalization bound.} The PAC-Bayesian framework provides guarantees on the expected risk of a randomized predictor (hypothesis) that depends on the training dataset. The hypothesis is drawn from a distribution $Q$ and sometimes referred to as a posterior given a training algorithm. We then denote the expected risk with respect to the distribution $Q$ as $R(Q)$ and the empirical risk with respect to the distribution $Q$ as $\hat{R}(Q)$. Note that $P$ is typically assumed to be a Gaussian prior, $\mathcal{N}(0, \lambda^{-1}I)$, over the weight space $\Theta$, where $\lambda$ is the $L_{2}$ regularization strength \citep{graves2011practical,neyshabur2017exploring,he2019control}. 

\begin{assumption}
 \label{as:gaussprior}
The prior over model weights is Gaussian, $P = \mathcal{N}(0, \lambda^{-1}I)$.
\end{assumption}

We introduce the classical PAC-Bayesian generalization bound in Theorem \ref{pr:pac}.

 \begin{theorem}[The PAC-Bayesian Generalization Bound \citep{mcallester1999some}]
 \label{pr:pac}
For any real $\Delta \in (0,1)$, with probability at least $1-\Delta$, over the draw of the training dataset $S$, the expected risk for all distributions $Q$ satisfies
\[  R(Q) - \hat{R}(Q) \leq 4\sqrt{\frac{1}{N} [\KL(Q \| P) + \ln(\frac{2N}{\Delta})]}   , \]
where $\KL(Q \| P)$ denotes the Kullback–Leibler divergence from $P$ to $Q$. 
\end{theorem}
We define that $\Gen(Q) = R(Q) - \hat{R}(Q)$ is the expected generalization gap. Then the upper bound of $\Gen(Q)$ can be written as
\begin{align}
\label{eq:pacsup}
\Sup \Gen(Q) = 4\sqrt{\frac{1}{N} [\KL(Q \| P) + \ln(\frac{2N}{\Delta})]} .
\end{align}
Theorem \ref{pr:pac} demonstrates that the expected generalization gap's upper bound closely depends on the posterior given by a training algorithm. Thus, it is possible to improve generalization by decreasing $\KL(Q \| P)$. 

\textbf{The Kullback–Leibler divergence from $P$ to $Q$.} Note that $Q$ and $P$ are two Gaussian distributions, $\mathcal{N}(\mu_{Q}, \Sigma_{Q})$ and $\mathcal{N}(\mu_{P}, \Sigma_{P})$, respectively. Then we have  
\begin{align}
\label{eq:pqkldiv}
 \KL(Q \| P) 
=  & \frac{1}{2} \left[ \log \frac{\det(\Sigma_{P})}{\det(\Sigma_{Q})}  + \Tr (\Sigma_{P}^{-1} \Sigma_{Q})  \right]  + \nonumber \\ 
  & \frac{1}{2}  (\mu_{Q} - \mu_{P})^{\top} \Sigma_{P}^{-1} (\mu_{Q} - \mu_{P}) - \frac{n}{2}.
\end{align}
Suppose that $\mu_{Q} = \theta^{\star}$ and $\Sigma_{Q}(\gamma) = \gamma \Sigma_{sgd} $ correspond to the covariance-rescaled SGD posterior, and $\mu_{P} = 0$ and $\Sigma_{P} = \lambda^{-1} I$ correspond to $L_{2}$ regularization. Here $\gamma = ((1+\beta_{0})^{2} + \beta_{0}^{2}) \geq 1$. Then $ \KL(Q(\gamma) \| P) $ is a function of $\gamma$, written as 
\begin{align}
\label{eq:priorposteriorkldiv}
 \KL(Q(\gamma) \| P) 
=  & \frac{1}{2}\log \frac{\lambda^{-n}}{\gamma^{n}\det(\Sigma_{sgd})}  +   \frac{1}{2}\lambda^{-1}\gamma \Tr (\Sigma_{sgd})  \nonumber \\ 
    & + \frac{\lambda}{2} \| \theta^{\star} \|^{2} - \frac{n}{2}.
\end{align}
For minimizing $\KL(Q(\gamma) \| P)$, we calculate its gradient with respect to $\gamma$ as
\begin{align}
\label{eq:kldivgrad}
\nabla_{\gamma} \KL(Q(\gamma) \| P) =  \frac{n}{2} (\frac{\eta}{2B\lambda} - \frac{1}{\gamma}),
\end{align}
where we have used Equation \eqref{eq:sgdsigma}. 

It shows that $\KL(Q(\gamma) \| P)$ is a monotonically decreasing function on the interval of $\gamma \in [1, \frac{2B\lambda}{\eta}]$. 
Obviously, when  $\frac{\eta}{2B\lambda}<1$ holds, we can always decrease $\KL(Q(\gamma) \| P)$ by fine-tuning $\gamma > 1$. Note that Momentum is a special case of PNM with $\beta_{0} = - \frac{\beta_{1}}{1+\beta_{1}}$, and the Momentum posterior is approximately equivalent to the vanilla SGD posterior.

\textbf{Stochastic PNM may have a tighter bound than SGD.}  Based on the results above, we naturally prove Theorem \ref{pr:pnmgeneralization} that Stochastic PNM can always have a tighter upper bound of the generalization gap than SGD by fine-tuning $\beta_{0} > 0$ in any task where $\frac{\eta}{2B\lambda}<1$ holds for SGD. In principle, SHB may be reduced to SGD with a different learning rate. SGD in PyTorch is actually equivalent to SHB. Thus, our theoretical analysis can be easily generalized to SHB.

 \begin{theorem}[The generalization advantage of Stochastic PNM]
 \label{pr:pnmgeneralization}
Suppose Assumption \ref{as:gaussprior}, the conditions of Theorem \ref{pr:sgdposterior}, and Theorem \ref{pr:pac} hold. If $\frac{\eta}{2B\lambda}<1$ holds for SGD, then there must exist $\beta_{0} \in (0, \frac{2B\lambda}{\eta}]$ that makes the following hold in the long-time limit:
\begin{align*}
\Sup \Gen(Q_{\mathrm{pnm}}) < \Sup \Gen(Q_{sgd}) .
\end{align*}
\end{theorem}

\textbf{When does $\frac{\eta}{2B\lambda} < 1$ hold?} We do not theoretically prove that $\frac{\eta}{2B\lambda} < 1$ always holds. However, it is very easy to empirically verify the inequality $\frac{\eta}{2B\lambda} <1$ in any specific practical task. Fortunately, we find that $\frac{\eta}{2B\lambda} < 1$ holds in wide practical applications. For example, we have $\frac{\eta}{2B\lambda} \approx 0.039 $ for the common setting that $\eta=0.001$ (after the final learning rate decay), $B=128$, and $\lambda=1e-4$. It means that the proposed PNM method may improve generalization in wide applications. 

\textbf{How to select $\beta_{0}$ in practice?} Recent work \citep{he2019control} suggested that increasing $\frac{\eta}{B}$ always improves generalization by using the PAC-Bayesian framework similarly to our work. However, as we discussed in Section \ref{sec:intro}, this is not true in practice for multiple reasons. 

Similarly, while Equation \eqref{eq:kldivgrad} suggests that $\gamma = \frac{2B\lambda}{\eta}$ can minimize $\Sup \Gen(Q_{\mathrm{pnm}})$ in principle, we do not think that $\gamma = \frac{2B\lambda}{\eta}$ should always be the optimal setting in practice. Instead, we suggest that a $\gamma$ slightly larger than one is good enough in practice. In our paper, we choose $\gamma=5$ as the default setting, which corresponds to $\beta_{0} = 1$. 

We do not choose $\gamma = \frac{2B\lambda}{\eta}$ mainly because a too large $\gamma$ requires too many iterations to reach the long-time limit. Theorem \ref{pr:pnmconverge} also demonstrates that it will require much more iterations to reach convergence if $\beta_{0}$ is too large. However, in practice, the number of training epochs is usually fixed and is not consider as a fine-tuned hyperparameter. This supports our belief that PNM with $\gamma$ slighter larger than one can be a good and robust default setting without re-tuning the hyperparameters.

\section{Empirical Analysis}
\label{sec:empirical}

\begin{figure}
\center
\subfigure{\includegraphics[width =0.49\columnwidth ]{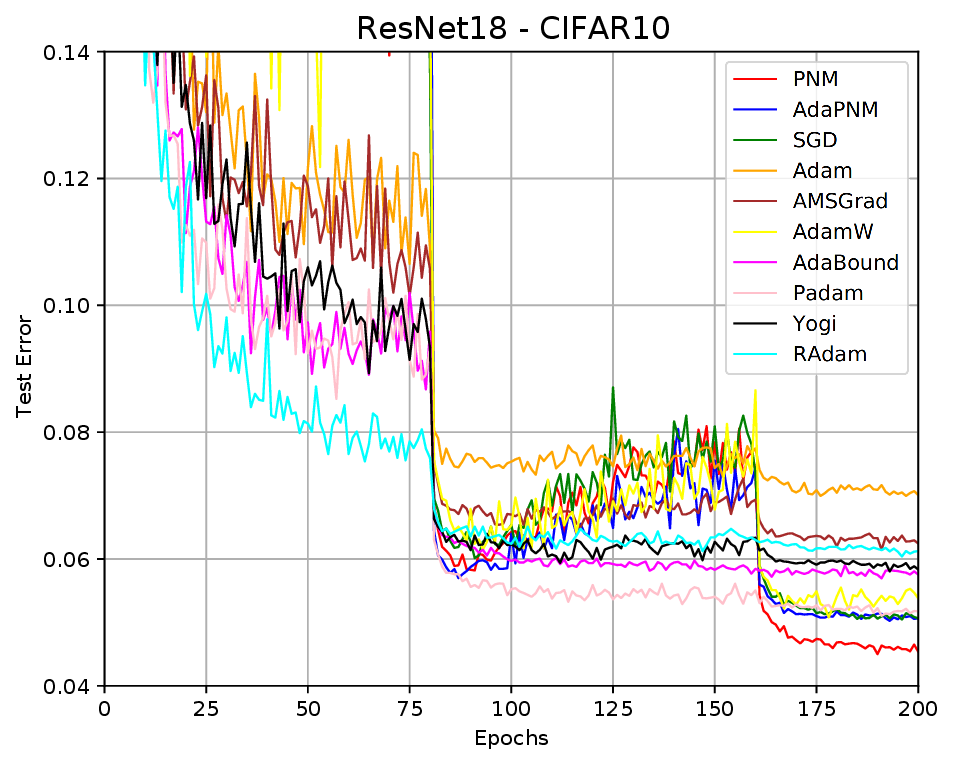}} 
\subfigure{\includegraphics[width =0.49\columnwidth ]{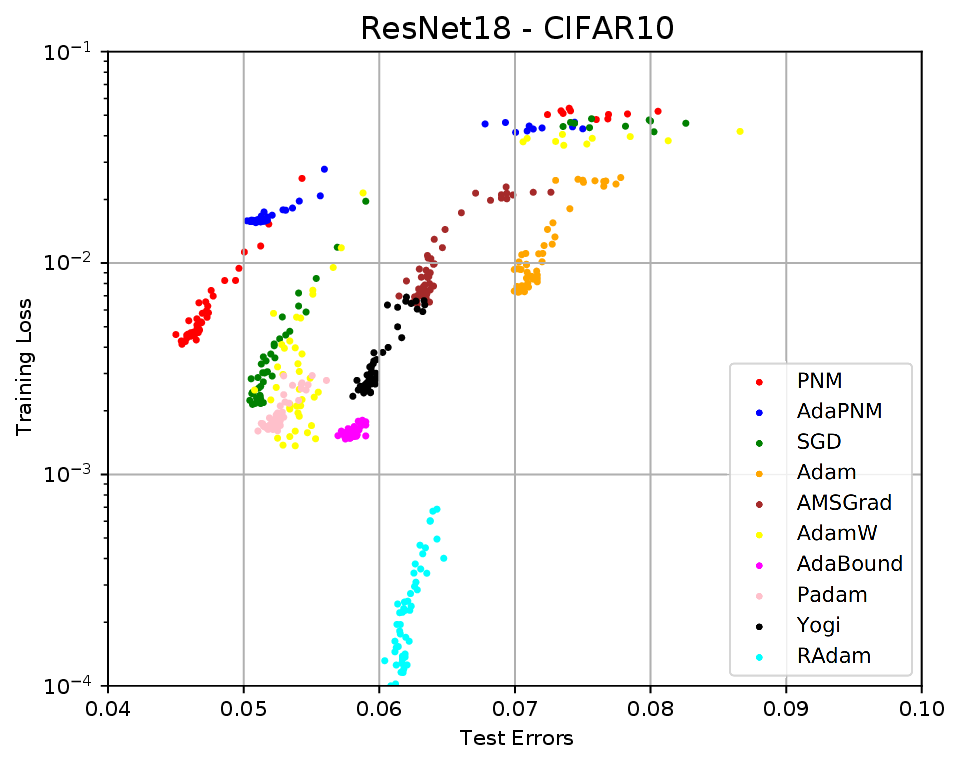}}
\subfigure{\includegraphics[width =0.49\columnwidth ]{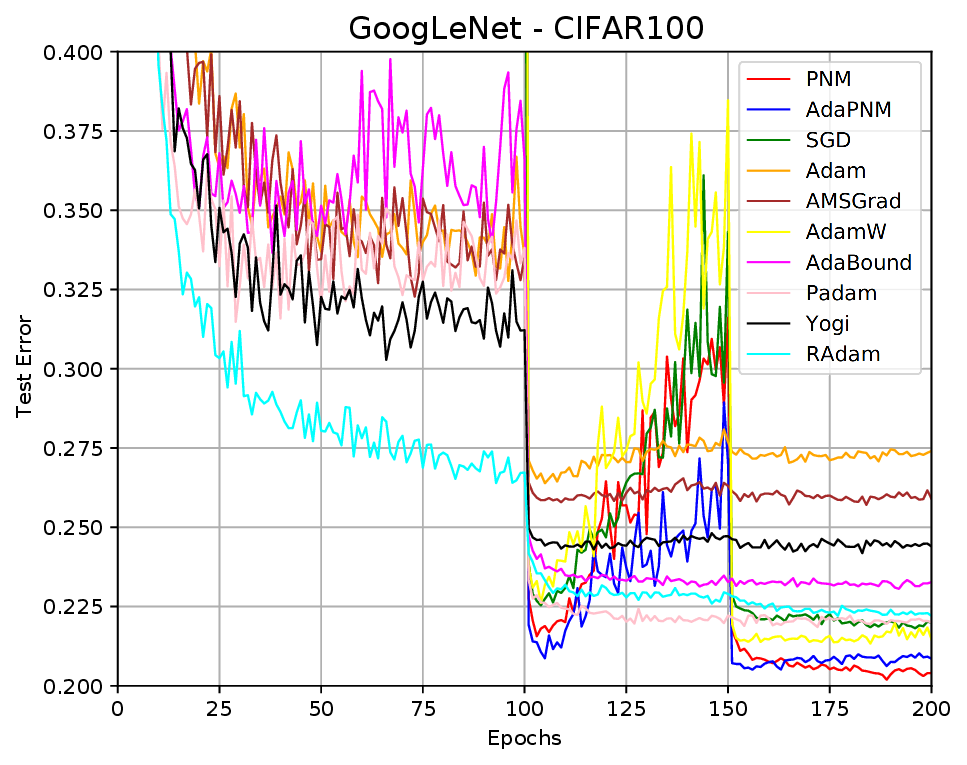}}  
\subfigure{\includegraphics[width =0.49\columnwidth ]{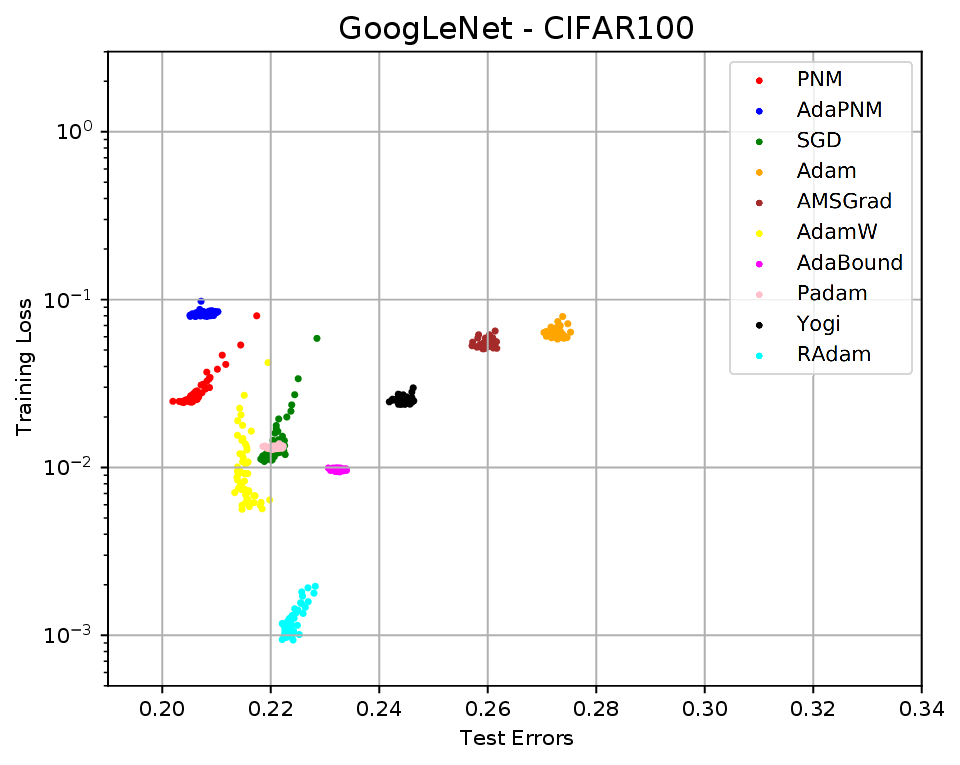}}  
\caption{ The learning curves of popular models on CIFAR-10 and CIFAR-100, respectively. Left Column: Test curves. Right Column: The scatter plots of training losses and test errors during final 40 epochs. It demonstrates that PNM and AdaPNM yield significantly better test results.}
 \label{fig:cifarpnm}
\end{figure}

In this section, we empirically study how the PNM-based optimizers are compared with conventional optimizers. 

\textbf{Models and Datasets.} We trained popular deep models, including ResNet18/ResNet34/ResNet50 \citep{he2016deep}, VGG16 \citep{simonyan2014very}, DenseNet121 \citep{huang2017densely}, GoogLeNet \citep{szegedy2015going}, and Long Short-Term Memory (LSTM) \citep{hochreiter1997long} on CIFAR-10/CIFAR-100 \citep{krizhevsky2009learning}, ImageNet \citep{deng2009imagenet} and Penn TreeBank \citep{marcus1993building}. We leave the implementation details in Appendix \ref{sec:appexperiment}.

\begin{figure}
\centering
\subfigure{\includegraphics[width =0.49\columnwidth ]{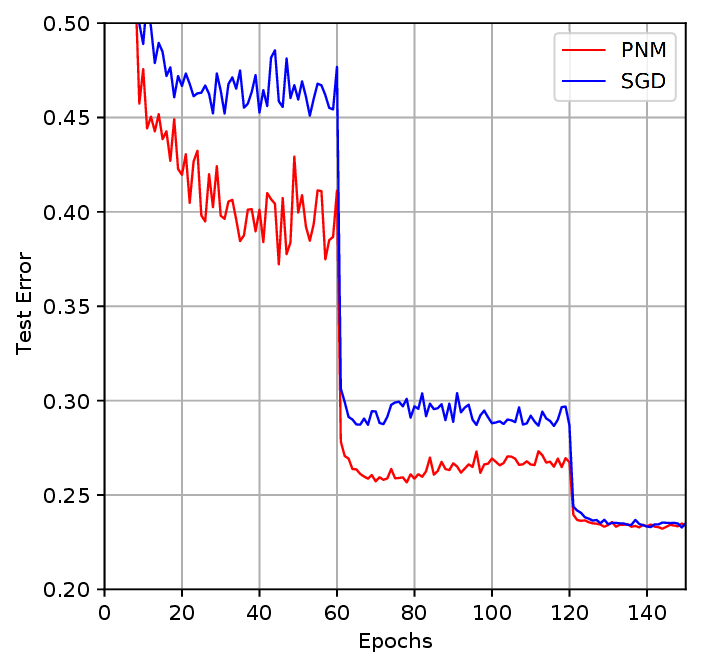}}
\subfigure{\includegraphics[width =0.49\columnwidth ]{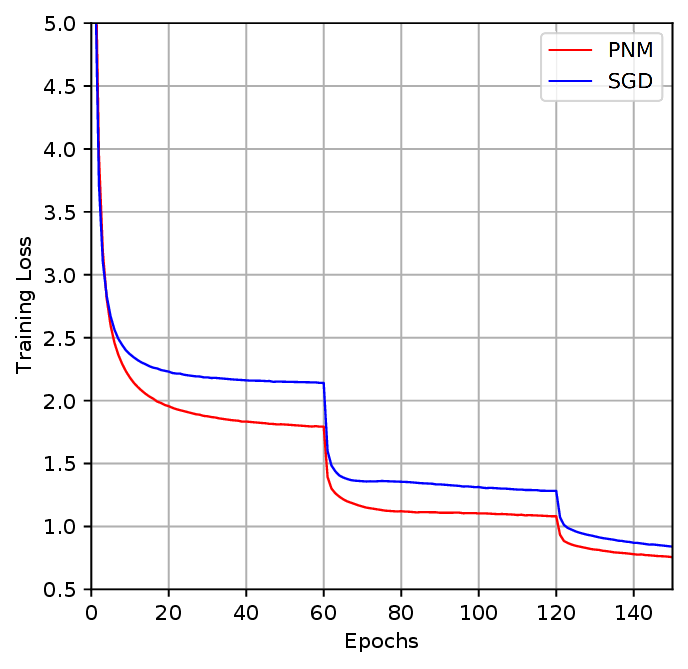}}  
\begin{tabular}{lllll}
\toprule
           & Epoch60 & Epoch120 & Epoch150    \\
\midrule
PNM(Top1)  &  $\mathbf{37.23}$ & $\mathbf{25.67}$  & $\mathbf{23.21}$  \\
SGD(Top1)  & $45.10$ & $28.66$ & $23.28$  \\
\midrule
PNM(Top5)  &  $\mathbf{14.53}$ & $\mathbf{7.79}$  & $\mathbf{6.73}$  \\
SGD(Top5)  & $19.04$ & $9.30$ & $6.75$  \\
\midrule
\bottomrule
\end{tabular}
\caption{ The learning curves of ResNet50 on ImageNet. Left Subfigure: Top 1 Test Error. Right Subfigure: Training Loss. PNM not only generalizes significantly better than conventional momentum, but also converges faster. PNM always achieves lower training losses and test errors at the final epoch of each learning rate decay phase.}
 \label{fig:imagenetpnm}
\end{figure}

 \begin{table}
\caption{Top-1 and top-5 test errors of ResNet50 on ImageNet. Note that the popular SGD baseline performace of ResNet50 on ImageNet has the test errors as $23.85\%$ in PyTorch and $24.9\%$ in \citet{he2016deep}, which are both worse than our SGD baseline. AdaPNM (with decoupled weight decay and no amsgrad) significantly outperforms its conventional variant, Adam (with decoupled weight decay and no amsgrad).}
\label{table:imagenet}
\begin{center}
\begin{small}
\begin{sc}
\begin{tabular}{lll | ll}
\toprule
           & PNM & SGD & AdaPNM & AdamW   \\
\midrule
Top1  &  $\mathbf{23.21}$ & $23.28$  &  $\mathbf{23.12}$ & $23.62$ \\
Top5  & $\mathbf{6.73}$ & $6.75$ & $\mathbf{6.82}$   & $7.09$  \\
\midrule
\bottomrule
\end{tabular}
\end{sc}
\end{small}
\end{center}
\end{table}

\begin{figure}
\centering
\subfigure{\includegraphics[width=0.49\columnwidth]{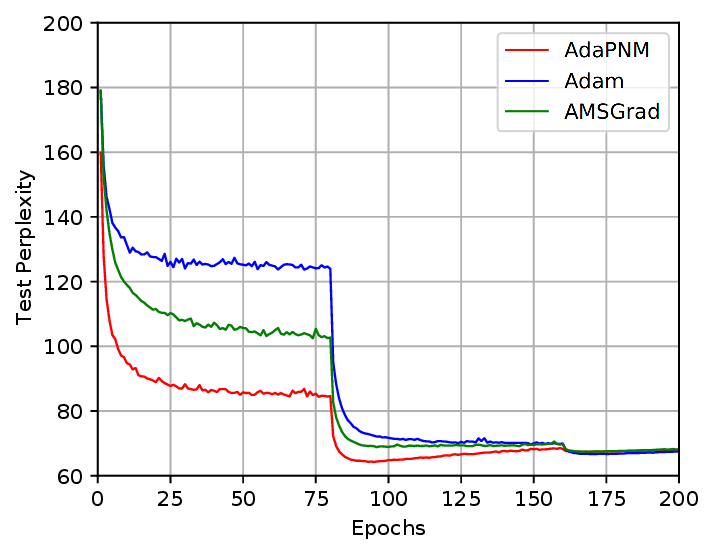}} 
\subfigure{\includegraphics[width=0.49\columnwidth]{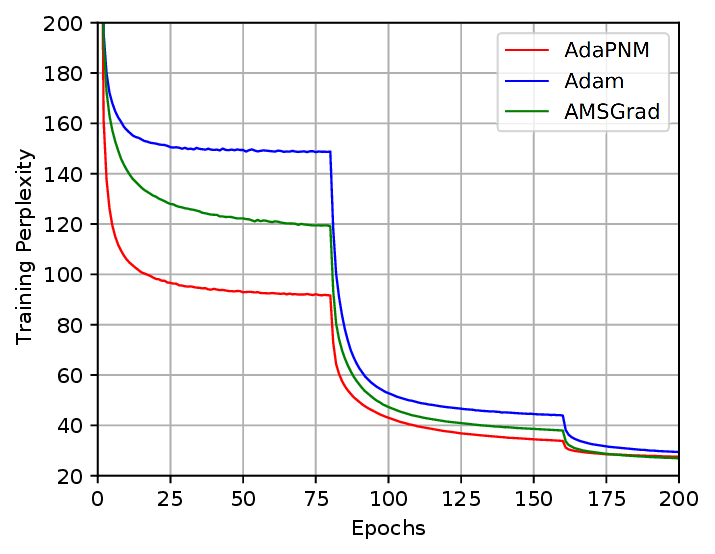}} 
\caption{The learning curves of LSTM on Penn Treebank. The optimal test perplexity of AdaPNM, Adam, and AMSGrad are $64.25$, $66.67$, and $67.40$, respectively. AdaPNM not only converges much faster than Adam and AMSGrad, but also yields lower test perplexity.}
 \label{fig:lstmadapnm}
\end{figure}

\begin{figure}
\centering
\subfigure{\includegraphics[width=0.49\columnwidth]{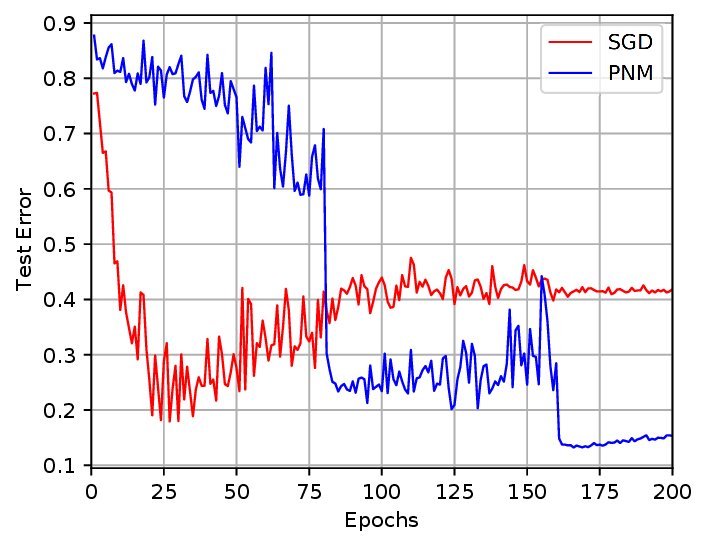}} 
\subfigure{\includegraphics[width=0.49\columnwidth]{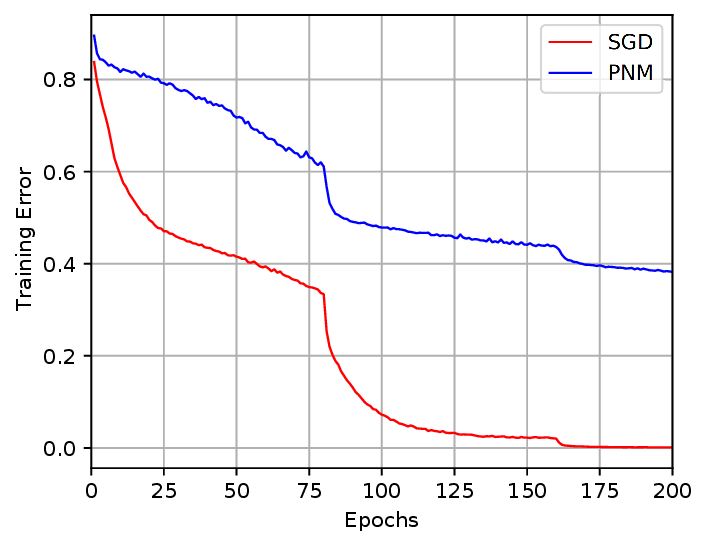}} \\
\subfigure{\includegraphics[width=0.49\columnwidth]{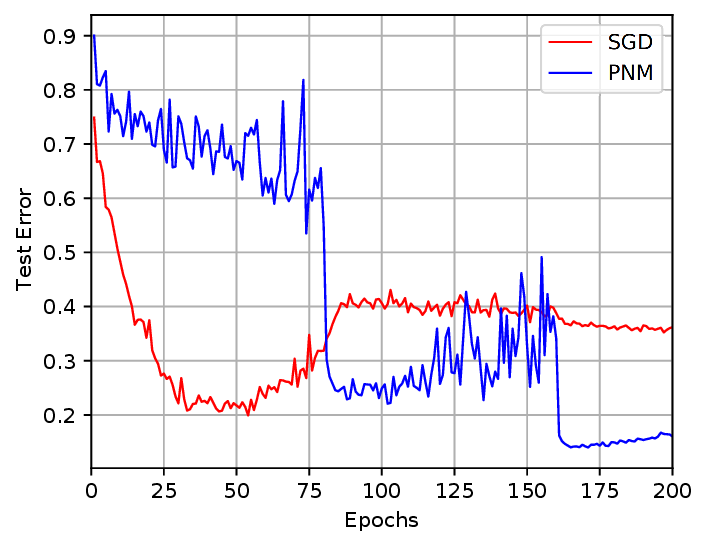}} 
\subfigure{\includegraphics[width=0.49\columnwidth]{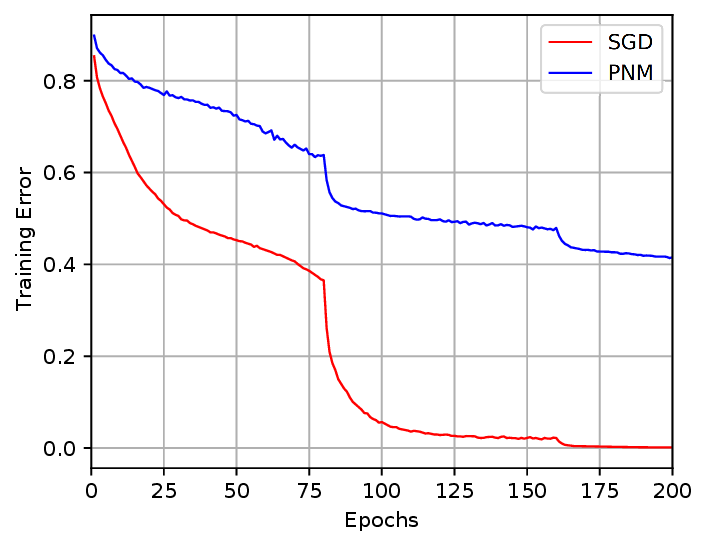}}
\caption{We compare PNM and SGD (with Momentum) by training ResNet34 on CIFAR-10 with $40\%$ asymmetric label noise (Top Row) and $40\%$ symmetric label noise (Bottom Row). Left: Test Curve. Right: Training Curve. We observe that PNM with a large $\beta_{0}$ may effectively relieve memorizing noisy labels and almost only learn clean labels, while SGD almost memorizes all noisy labels and has a nearly zero training error.}
 \label{fig:labelnoisepnm}
\end{figure}

\textbf{Image Classification on CIFAR-10 and CIFAF-100.}  In Table \ref{table:cifar}, we first empirically compare PNM and AdaPNM with popular stochastic optimizers, including SGD, Adam \citep{kingma2014adam}, AMSGrad \citep{reddi2019convergence}, AdamW \citep{loshchilov2018decoupled}, AdaBound \citep{luo2019adaptive}, Padam \citep{chen2018closing}, Yogi \citep{zaheer2018adaptive}, and RAdam \citep{liu2019variance} on CIFAR-10 and CIFAR-100. It demonstrates that PNM-based optimizers generalize significantly better than the corresponding conventional Momentum-based optimizers. In Figure \ref{fig:cifarpnm}, we note that PNM-based optimizers have better test performance even with similar training losses. 

\textbf{Image Classification on ImageNet.} Table \ref{table:imagenet} also supports that the PNM-based optimizers generalizes better than the corresponding conventional Momentum-based optimizers. Figure \ref{fig:imagenetpnm} shows that, on the experiment on ImageNet, Stochastic PNM consistently has lower test errors than SGD at the final epoch of each learning rate decay phase. It indicates that PNM not only generalizes better, but also converges faster on ImageNet due to stronger SGN. 

\textbf{Language Modeling.} As Adam is the most popular optimizer on Natural Language Processing tasks, we further compare AdaPNM (without amsgrad) with Adam (without amsgrad) as the baseline on the Language Modeling experiment. Figure \ref{fig:lstmadapnm} shows that AdaPNM outperforms the conventional Adam in terms of test performance and training speed. 

\textbf{Learning with noisy labels.} Deep networks can easily overfit training data even with random labels \citep{zhang2017understanding}. We run PNM and SGD on CIFAR-10 with $40\%$ label noise for comparing the robustness to noise memorization. Figure \ref{fig:labelnoisepnm} shows that PNM has much better generalization than SGD and outperforms SGD by more 20 points at the final epoch. It also demonstrates that enhancing SGN may effectively mitigate overfitting training data.

\begin{figure}
\centering
\subfigure{\includegraphics[width=0.49\columnwidth]{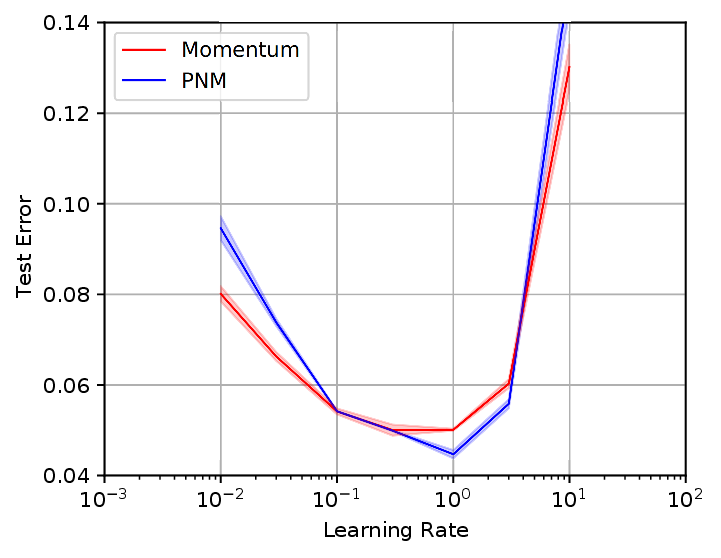}}
\subfigure{\includegraphics[width=0.49\columnwidth]{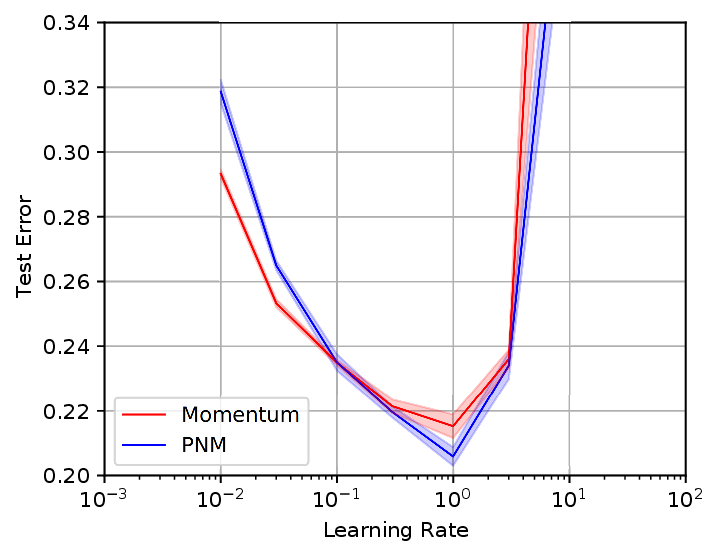}} 
        \caption{ We compare the generalization of PNM and SGD under various learning rates. Left: ResNet18 on CIFAR-10. Right: ResNet34 on CIFAR-100. The optimal performance of PNM is better than the conventional SGD.}
 \label{fig:lrresnetpnm}
\end{figure}

\begin{figure}
\centering
\subfigure{\includegraphics[width=0.9\columnwidth]{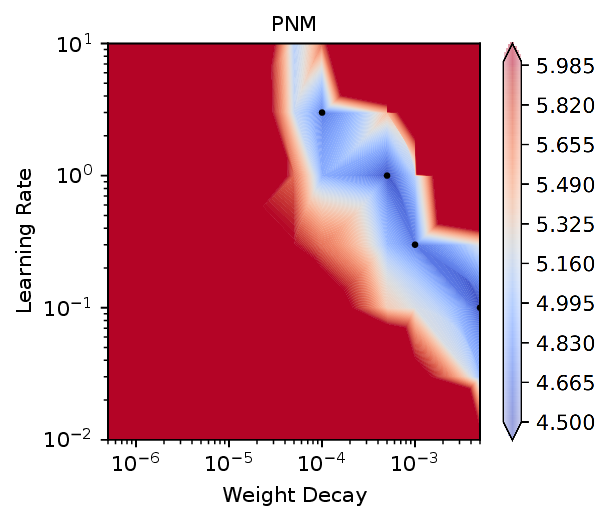}}
\subfigure{\includegraphics[width=0.9\columnwidth]{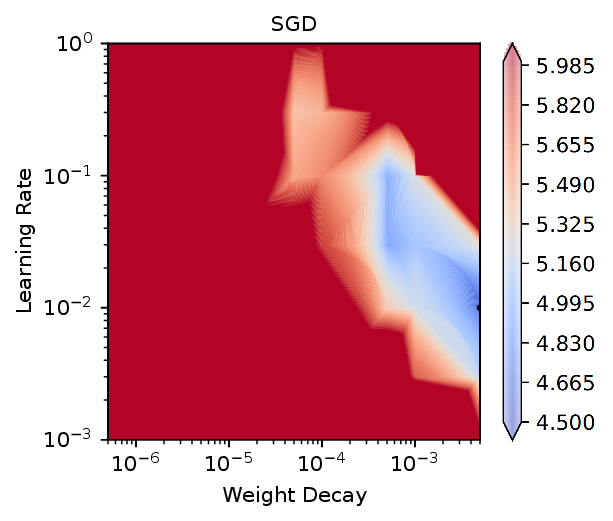}}
\caption{The test errors of ResNet18 on CIFAR-10 under various learning rates and weight decay. PNM has a much deeper and wider blue region near dark points ($\leq 4.83\%$) than SGD.}
 \label{fig:lrwdcifarpnm}
\end{figure}

\begin{figure}
\centering
\subfigure{\includegraphics[width=0.9\columnwidth]{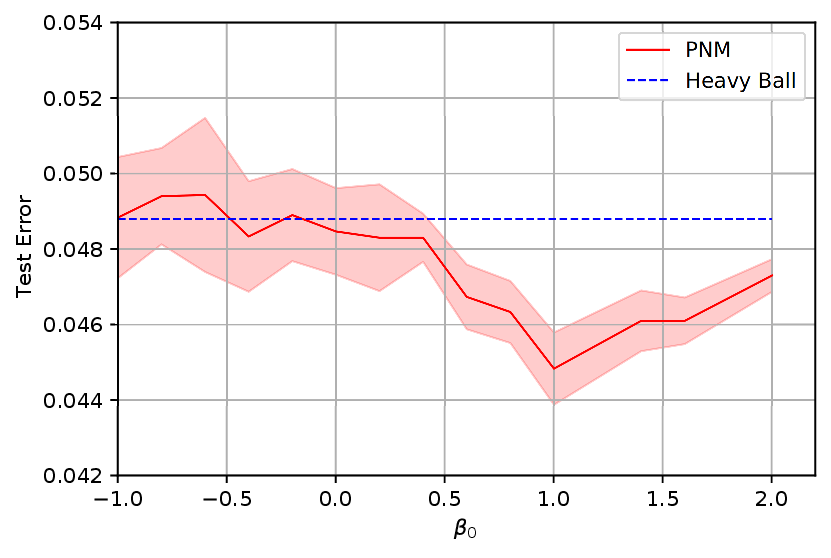}}
       \caption{ We train ResNet18 on CIFAR-10 under various $\beta_{0}$ choices. It demonstrates that PNM may achieve significantly better generalization with a wide range of $\beta_{0} > 0$, which corresponds to a positive-negative momentum pair for enhancing SGN as we expect. With any $\beta_{0} \in [-1, 0]$, the test performance does not sensitively depend on $\beta_{0}$, because this case cannot enhance SGN. }
 \label{fig:betapnm}    
\end{figure}

\textbf{Robustness to the learning rate and weight decay.} In Figure \ref{fig:lrresnetpnm}, we show that PNM can consistently outperform SGD under a wide range of learning rates. Figure \ref{fig:lrwdcifarpnm} further supports that PNM can be more robust to learning rates and weight decay than SGD, because PNM has a significantly deeper and wider basin in terms of test errors. This makes PNM a robust alternative to conventional Momentum. 

\textbf{Robustness to the new hyperparameter $\beta_{0}$.} Finally, we empirically study how PNM depends on the hyperparameter $\beta_{0}$ in practice in Figure \ref{fig:betapnm}. The result in Figure \ref{fig:betapnm} fully supports our motivation and theoretical analysis. It demonstrates that PNM may achieve significantly better generalization by choosing a proper $\beta_{0} > 0$, which corresponds to a positive-negative momentum pair for enhancing SGN as we expect. With any $\beta_{0} \in [-1, 0]$, the test performance does not sensitively depend on $\beta_{0}$. Because the case that $\beta_{0} \in [-1, 0]$ corresponds to a positive-positive momentum pair and, thus, cannot enhance SGN. 

\textbf{Supplementary experiments.} Please refer to Appendix \ref{sec:ablation}.

\section{Conclusion}
\label{sec:conclusion}

We propose a novel Positive-Negative Momentum method for manipulating SGN by using the difference of past gradients. The simple yet effective method can provably improve deep learning at very low costs. In practice, the PNM method is a powerful and robust alternative to the conventional Momentum method in classical optimizers and can usually make significant improvements. 

While we only use SGD and Adam as the conventional base optimizers, it is easy to incorporate PNM into other advanced optimizers. Considering the importance and the popularity of Momentum, we believe that the proposed PNM indicates a novel and promising approach to designing optimization dynamics by manipulating gradient noise.

\section*{Acknowledgement}
We thank Dr. Jinze Yu for his helpful discussion. MS was supported by the International Research Center for Neurointelligence (WPI-IRCN) at The University of Tokyo Institutes for Advanced Study.

\bibliography{deeplearning}
\bibliographystyle{icml} 

\appendix

\onecolumn 
 
\section{Proofs}
\label{sec:appproofs}

\subsection{Proof of Theorem \ref{pr:pnmconverge}}
We first propose several useful lemmas. Note that $\theta_{-2} = \theta_{-1} = \theta_{0}$. 
 \begin{lemma}
 \label{pr:convergencelm1}
Let $z_{t} = \frac{x_{t} - \beta x_{t-2}}{1- \beta}$. Under the conditions of Theorem \ref{pr:pnmconverge}, for any $t \geq 0$, we have
\begin{align*}
z_{t+1} - z_{t} = - \frac{\alpha}{1-\beta} g_{t}.
\end{align*}
 \end{lemma}

\begin{proof}
Recall that 
\begin{align}
x_{t+1} = x_{t} - \alpha g_{t} + \beta (x_{t-1} - x_{t-2}).
\end{align}
Then we have
\begin{align}
x_{t+1} &= x_{t} - \alpha g_{t} + \beta (x_{t-1} - x_{t-2}) \\
x_{t+1}  - \beta x_{t-1} &=  x_{t}  - \beta x_{t-2} - \alpha g_{t} \\
\frac{x_{t+1}  - \beta x_{t-1}}{1-\beta} &= \frac{ x_{t}  - \beta x_{t-2}}{1- \beta} - \frac{\alpha }{1- \beta}g_{t} \\
z_{t+1} - z_{t} &= - \frac{\alpha}{1-\beta} g_{t}.
\end{align}
The proof is now complete.
\end{proof}

 \begin{lemma}
 \label{pr:convergencelm2}
Under the conditions of Theorem \ref{pr:pnmconverge}, for any $t\geq 0$, we have
\begin{align*}
\mathbb{E}[ f(z_{t+1}) - f(z_{t}) ] \leq  & \frac{1}{2L} \mathbb{E}[ \| f(z_{t}) -f(\theta_{t})\|^{2}] + \\
                                                      & \left( \frac{L\alpha^{2}}{(1-\beta)^{2}} - \frac{\alpha}{1- \beta} \right) \mathbb{E}[\|\nabla f(\theta_{k}) \|^{2}] + \frac{L\alpha^{2} \sigma^{2}}{2(1-\beta)^{2}} . 
\end{align*}
 \end{lemma}
 
\begin{proof}
As $f(\theta)$ is $L$-smooth, we have
\begin{align}
f(z_{t+1}) \leq f(z_{t}) + \nabla f(z_{t})^{\top}(z_{t+1} - z_{t}) + \frac{L}{2} \| z_{t+1} - z_{t} \|^{2}.
\end{align}
By Lemma \ref{pr:convergencelm1}, we obtain
\begin{align}
f(z_{t+1}) \leq f(z_{t}) - \frac{\alpha}{1-\beta} \nabla f(z_{t})^{\top}g_{t} + \frac{L\alpha^{2}}{2(1-\beta)^{2}} \|g_{t}\|^{2} .
\end{align}
Recall that $g_{t} = \nabla f(\theta_{t}) + \xi_{t}$ for stochastic optimization. Then 
\begin{align*}
f(z_{t+1})  \leq & f(z_{t}) - \frac{\alpha}{1-\beta} \nabla f(z_{t})^{\top} (\nabla f(\theta) + \xi_{t}) + \frac{L\alpha^{2}}{2(1-\beta)^{2}} \|\nabla f(\theta_{t}) + \xi_{t} \|^{2}  \\
       =  & f(z_{t}) - \frac{\alpha}{1-\beta} \nabla f(z_{t})^{\top} (\nabla f(\theta) + \xi_{t}) + \frac{L\alpha^{2}}{2(1-\beta)^{2}} \|\nabla f(\theta_{t}) + \xi_{t} \|^{2} \\
       = & f(z_{t}) - \frac{\alpha}{1-\beta} \nabla f(z_{t})^{\top} \xi_{t}  - \frac{\alpha}{1-\beta} ( \nabla f(z_{t}) -  \nabla f(\theta_{t}) )^{\top} \nabla f(\theta_{t}) - \\
            &   \frac{\alpha}{1-\beta}\| \nabla f(\theta_{t}) \|^{2} + \frac{L\alpha^{2}}{2(1-\beta)^{2}} \|\nabla f(\theta_{t}) + \xi_{t} \|^{2} \\
\end{align*}
Recall that $\mathbb{E}[\xi_{t}] =0$ and $\mathbb{E}[\| g(\theta, \xi) - \nabla f(\theta) \|] \leq \sigma^{2} $. We take expectation on both sides, and obtain
\begin{align*}
\mathbb{E}[ f(z_{t+1}) - f(z_{t})]  \leq &  - \frac{\alpha}{1-\beta} \mathbb{E}[ ( \nabla f(z_{t}) -  \nabla f(\theta_{t}) )^{\top} \nabla f(\theta_{t})]  + \\
                                                          &  \left(\frac{L\alpha^{2}}{2(1-\beta)^{2}} -  \frac{\alpha}{1-\beta} \right) \mathbb{E}[\| \nabla f(\theta_{t}) \|^{2}] + \frac{L\alpha^{2} \sigma^{2}}{2(1-\beta)^{2}} 
\end{align*}
By the Cauchy-Schwarz Inequality, we obtain
\begin{align*}
\mathbb{E}[ f(z_{t+1}) - f(z_{t})]  \leq &  \mathbb{E}\left[  \frac{1}{2L}\| \nabla f(z_{t}) -  \nabla f(\theta_{t})\|^{2} + \frac{L\alpha^{2}}{2(1-\beta)^{2}} \| \nabla f(\theta_{t})\|^{2} \right]  + \\
                                                          &  \left(\frac{L\alpha^{2}}{2(1-\beta)^{2}} -  \frac{\alpha}{1-\beta} \right) \mathbb{E}[\| \nabla f(\theta_{t}) \|^{2}] + \frac{L\alpha^{2} \sigma^{2}}{2(1-\beta)^{2}} \\
                                                       = &  \frac{1}{2L} \mathbb{E}[ \| f(z_{t}) -f(\theta_{t})\|^{2}] + \\
                                                      & \left( \frac{L\alpha^{2}}{(1-\beta)^{2}} - \frac{\alpha}{1- \beta} \right) \mathbb{E}[\|\nabla f(\theta_{k}) \|^{2}] + \frac{L\alpha^{2} \sigma^{2}}{2(1-\beta)^{2}}
\end{align*}
The proof is now complete.
\end{proof}

 \begin{lemma}
 \label{pr:convergencelm3}
Under the conditions of Theorem \ref{pr:pnmconverge}, for any $t\geq 0$, we have
\begin{align*}
\mathbb{E}[  \| \nabla f(z_{t}) - \nabla f(\theta_{t}) \|^{2} ] \leq  \frac{L^{2}\alpha^{2} [\beta + \beta_{0}(1-\beta) ]^{2} (G^{2} + \sigma^{2}) }{(1-\beta)^{4}} . 
\end{align*}
 \end{lemma}
 
\begin{proof}
As $f$ is $L$-smooth, we have
\begin{align*}
\| \nabla f(z_{t}) - \nabla f(\theta_{t})\|^{2} \leq L^{2} \|z_{t} - \theta_{t} \|^{2}.
\end{align*}
Recall that $z_{t} = \frac{x_{t} -\beta x_{t-2}}{1-\beta}$ and $x_{t} = \theta_{t} +  \frac{\alpha \beta_{0}}{1-\beta} m_{t-1}$. Then 
\begin{align*}
\| \nabla f(z_{t}) - \nabla f(\theta_{t})\|^{2} \leq & L^{2} \|z_{t} - \theta_{t} \|^{2} \\
                                                                  = & L^{2} \| \frac{x_{t} - \beta x_{t-2} }{1-\beta} - \theta_{t} \| \\
                                                                   = & L^{2} \| \frac{\theta_{t} - \beta \theta_{t-2} }{1-\beta} + \frac{\alpha \beta_{0}}{1-\beta} g_{t-1} - \theta_{t} \| \\
                                                                   = & \frac{L^{2} }{ (1-\beta)^{2}} \| \beta (\theta_{t} - \theta_{t-2}) + \alpha\beta_{0} g_{t-1} \| .
\end{align*}
Without loss of generality, we assume $t$ is an even number. Recalling that the updating rule of $x_{t}$, we have
\begin{align*}
\theta_{t} - \theta_{t-2} = & \beta (\theta_{t-2} - \theta_{t-4}) - \alpha (g_{t-1} + g_{t-2}) \\
                                    = & - \alpha \sum_{k=0}^{\frac{t}{2}-1} \beta^{k} (g_{t-1-2k} + g_{t-2-2k}) .
\end{align*}
Then 
\begin{align*}
\| \nabla f(z_{t}) - \nabla f(\theta_{t})\|^{2} \leq & \frac{L^{2} \alpha^{2}}{ (1-\beta)^{2}} \| \beta \sum_{k=0}^{\frac{t}{2}-1} \beta^{k} (g_{t-1-2k} + g_{t-2-2k}) - \beta_{0} g_{t-1} \|^{2} .
\end{align*}
Let $\Gamma_{k} = \frac{1-\beta^{\frac{k}{2}} }{1-\beta}$. Then 
\begin{align*}
\| \nabla f(z_{t}) - \nabla f(\theta_{t})\|^{2} \leq & \frac{L^{2} \alpha^{2}}{ (1-\beta)^{2}} \| \beta \sum_{k=0}^{\frac{t}{2}-1} \beta^{k} (g_{t-1-2k} + g_{t-2-2k}) - \beta_{0} g_{t-1} \|^{2} .
\end{align*}
Similar to the proof of Lemma 4 in \citet{yan2018unified}, taking expectation on both sides gives
\begin{align*}
\mathbb{E} [\| \nabla f(z_{t}) - \nabla f(\theta_{t})\|^{2} ] \leq &  \frac{L^{2} \alpha^{2}}{ (1-\beta)^{2}}  (\beta \Gamma_{t} + \beta_{0} )^{2} (G^{2} + \sigma^{2})    \\
                                                                                       \leq &  \frac{L^{2} \alpha^{2}}{ (1-\beta)^{2}}  ( \frac{\beta}{1-\beta}+ \beta_{0} )^{2} (G^{2} + \sigma^{2}) .
\end{align*}
The proof is now complete.
\end{proof}

\begin{proof}
The proof of Theorem \ref{pr:pnmconverge} is organized as follows.

We first define $Q$ and $Q^{\prime}$ as 
\begin{align}
Q = \frac{\alpha}{1-\beta} - \frac{L\alpha^{2}}{(1-\beta)^{2}}
\end{align}
and 
\begin{align}
Q^{\prime} = \frac{L^{2} \alpha^{2} (\beta + \beta_{0}(1-\beta) )^{2} (G^{2} + \sigma^{2}) }{2(1-\beta)^{4}} + \frac{L\alpha^{2} \sigma^{2}}{2(1-\beta)^{2}}
\end{align}
By Lemma \ref{pr:convergencelm2} and Lemma \ref{pr:convergencelm3}, we have 
\begin{align}
\mathbb{E}[f(z_{k+1}) - f(z_{k})] \leq - Q \mathbb{E}[\| \nabla f(\theta_{k})\|^{2}] + Q^{\prime}
\end{align}
By summing the above inequalities for $k=0, \ldots, t $ and noting that $\alpha < \frac{1-\beta}{L}$, we have
\begin{align}
Q \sum_{k=0}^{t} \mathbb{E}[ \| \nabla f(\theta_{k})\|^{2}] \leq & \mathbb{E}[f(z_{0}) - f(z_{t+1})]  + (t+1)Q^{\prime} \\
                                                                                          \leq &  \mathbb{E}[f(z_{0}) - f^{\star}]  + (t+1)Q^{\prime}.
\end{align} 
Note that $z_{0} = \theta_{0}$. Then 
\begin{align}
\min_{k=0,\ldots, t} \mathbb{E}[ \| \nabla f(\theta_{k})\|^{2}] \leq \frac{f(\theta_{0}) - f^{\star} }{ (t+1) Q} + \frac{Q^{\prime}}{Q}.
\end{align} 
As $\alpha \leq \frac{1-\beta}{2L}$, we have $Q \geq \frac{\alpha}{2(1-\beta)} $. Then 
\begin{align}
\min_{k=0,\ldots, t} \mathbb{E}[ \| \nabla f(\theta_{k})\|^{2}] \leq \frac{ (f(\theta_{0}) - f^{\star} ) (1-\beta)}{ (t+1) \alpha} + \frac{2(1-\beta)}{\alpha} Q^{\prime}.
\end{align} 
As $ \frac{\eta}{ \sqrt{(1+\beta_{0})^{2} + \beta_{0}^{2}}} =  \frac{\alpha}{(1-\beta)} = \min \{ \frac{1}{2L}, \frac{C}{\sqrt{t+1}} \}$, we obtain
\begin{align}
\min_{k=0,\ldots, t} \mathbb{E}[ \| \nabla f(\theta_{k})\|^{2}] \leq & \frac{ 2(f(\theta_{0}) -f^{\star})}{t+1} \max \{ 2L, \frac{\sqrt{t+1}}{C} \} + \\ 
                                                                                           & \frac{ L \alpha (\beta + \beta_{0} (1-\beta))^{2} (G^{2} + \sigma^{2}) + L \alpha (1-\beta)^{2} \sigma^{2}}{(1-\beta)^{3}} \\
                                                                                           \leq & \frac{ 2(f(\theta_{0}) -f^{\star})}{t+1} \max \{ 2L, \frac{\sqrt{t+1}}{C} \} + \\ 
                                                                                           & \frac{C}{\sqrt{t+1}} \frac{ L (\beta + \beta_{0} (1-\beta))^{2} (G^{2} + \sigma^{2}) + L(1-\beta)^{2} \sigma^{2}}{(1-\beta)^{2}}
\end{align} 
The proof is now complete.
\end{proof}

\subsection{Proof of Theorem \ref{pr:pnmgeneralization}}

\begin{proof}
By Theorem \ref{pr:pac}, we first write the upper bound of the generalization gap for the Stochastic PNM posterior $Q(\gamma)$ as
\begin{align}
B(\gamma) = 4\sqrt{\frac{1}{N} [\KL(Q(\gamma) \| P) + \ln(\frac{2N}{\Delta})]}.
\end{align} 
Then we calculate the gradient of $B(\gamma)$ with respect to $\gamma$ as
\begin{align}
\nabla_{\gamma} B(\gamma) = 2 \left[\frac{1}{N} [\KL(Q(\gamma) \| P) + \ln(\frac{2N}{\Delta})] \right]^{-\frac{1}{2}} \frac{1}{N} \nabla_{\gamma}  \KL(Q(\gamma) \| P).
\end{align} 

By Assumption \ref{as:gaussprior} and Equation \eqref{eq:pnmsigma}, we have $ \KL(Q(\gamma) \| P)$ as Equation \eqref{eq:priorposteriorkldiv}. 

Then we have $\nabla_{\gamma}  \KL(Q(\gamma) \| P)$ as Equation \eqref{eq:kldivgrad}:
\begin{align*}
\nabla_{\gamma} \KL(Q(\gamma) \| P) =  \frac{n}{2} (\frac{\eta}{2B\lambda} - \frac{1}{\gamma}).
\end{align*}

Under the condition that $\frac{\eta}{2B\lambda} < 1$, we have
\begin{align}
\nabla_{\gamma} B(\gamma) = 2 \left[\frac{1}{N} [\KL(Q(\gamma) \| P) + \ln(\frac{2N}{\Delta})] \right]^{-\frac{1}{2}} \frac{1}{N} \frac{n}{2} (\frac{\eta}{2B\lambda} - \frac{1}{\gamma}) < 0
\end{align} 
for all $\gamma \in [1, \frac{2B\lambda}{\eta}]$. Note that $\gamma \in (1, \frac{2B\lambda}{\eta}]$ corresponds to Stochastic PNM and $\gamma = 1$ corresponds to SGD/Momentum.

It means that the upper bound $B(\gamma)$ is a monotonically decreasing function on the interval of $\gamma \in (1, \frac{2B\lambda}{\eta})$. 
We have
\begin{align}
B(\gamma) < B(1) 
\end{align} 
for $\gamma \in (1, \frac{2B\lambda}{\eta}]$. 

As $\Sup \Gen(Q_{\mathrm{pnm}})  = B(\gamma) $ and $\Sup \Gen(Q_{sgd}) = B(1)$, we have 
\begin{align}
\Sup \Gen(Q_{\mathrm{pnm}}) < \Sup \Gen(Q_{sgd}) 
\end{align} 
for $\gamma \in (1, \frac{2B\lambda}{\eta}]$.

The proof is now complete.

\end{proof}

\section{Implementation Details}
\label{sec:appexperiment}

\subsection{Image classification on CIFAR-10 and CIFAR-100}

\textbf{Data Preprocessing For CIFAR-10 and CIFAR-100:} We perform the common per-pixel zero-mean unit-variance normalization, horizontal random flip, and $32 \times 32$ random crops after padding with $4$ pixels on each side. 

\textbf{Hyperparameter Settings:} We select the optimal learning rate for each experiment from $\{0.0001, 0.001, 0.01, 0.1, 1, 10\}$ for PNM and SGD (with Momentum) and use the default learning rate for adaptive gradient methods. In the experiments on CIFAF-10 and CIFAR-100: $\eta=1$ for PNM; $\eta=0.1$ for SGD (with Momentum); $\eta=0.001$ for AdaPNM, Adam, AMSGrad, AdamW, AdaBound, and RAdam; $\eta=0.01$ for Padam. For the learning rate schedule, the learning rate is divided by 10 at the epoch of $\{80,160\}$ for CIFAR-10 and $\{100,150\}$ for CIFAR-100, respectively. The batch size is set to $128$ for both CIFAR-10 and CIFAR-100. 

The strength of weight decay is default to $\lambda = 0.0005$ as the baseline for all optimizers unless we specify it otherwise. Recent work \citet{xie2020stable} found that popular optimizers with $\lambda = 0.0005$ often yields test results than $\lambda = 0.0001$ on CIFAR-10 and CIFAR-100. Some related papers used $\lambda = 0.0001$ directly and obtained lower baseline performance than ours. We leave the empirical results with the weight decay setting $\lambda = 0.0001$ in Appendix \ref{sec:ablation}.

Note that decoupled weight decay which are used in AdamW has been rescaled by $1000$ times, which actually corresponds to $\lambda_{W} =0.5$. Following related papers, our work chose the type of weight decay for classical optimizers suggested by original papers. According to \citet{xie2020stable}, decoupled weight decay instead of conventional $L_{2}$ regularization is recommended in the presence of large gradient noise. We use decoupled weight decay in PNM and AdaPNM unless we specify it otherwise.

We set the momentum hyperparameter $\beta_{1} = 0.9$ for SGD with Momentum and PNM optimizers. As for other optimizer hyperparameters, we apply the default hyperparameter settings directly.  

We repeat each experiment for three times, and compute the standard deviations as error bars.

\subsection{Image classification on ImageNet}

\textbf{Data Preprocessing For ImageNet:} For ImageNet, we perform the per-pixel zero-mean unit-variance normalization, horizontal random flip, and the resized random crops where the random size (of 0.08 to 1.0) of the original size and a random aspect ratio (of $\frac{3}{4}$ to $\frac{4}{3}$) of the original aspect ratio is made. 

\textbf{Hyperparameter Settings for ImageNet:} We select the optimal learning rate for each experiment from $\{0.0001, 0.001, 0.01, 0.1, 1, 10\}$ for all tested optimizers. For the learning rate schedule, the learning rate is divided by 10 at the epoch of $\{60,120\}$. We train each model for $150$ epochs. The batch size is set to $256$. The weight decay hyperparameters are chosen as $\lambda = 0.0001$. We set the momentum hyperparameter $\beta_{1} = 0.9$ for SGD with Momentum and PNM optimizers. As for other optimizer hyperparameters, we still apply the default hyperparameter settings directly. 

\subsection{Language modeling}

We use a classical language model, Long Short-Term Memory (LSTM) \citep{hochreiter1997long} with 2 layers, 512 embedding dimensions, and 512 hidden dimensions, which has $14$ million model parameters and is similar to the ``medium LSTM'' in \citet{zaremba2014recurrent}. Note that our baseline performance is better than the reported baseline performance in \citet{zaremba2014recurrent}. The benchmark task is the word-level Penn TreeBank \citep{marcus1993building}.

\textbf{Hyperparameter Settings.} Batch Size: $B=20$. BPTT Size: $bptt=35$. We select the optimal learning rate for each experiment from $\{0.0001, 0.001, 0.01, 0.1, 1, 10\}$ for all tested optimizers. The weight decay hyperparameter is chosen as $\lambda=10^{-5}$. $L_{2}$ regularization applies to all tested optimizers. The dropout probability is set to $0.5$. We clipped gradient norm to $1$.

\subsection{Learning with noisy labels}

We trained ResNet34 via PNM and SGD (with Momentum) on corrupted CIFAR-10 with various asymmetric and symmetric label noise. The symmetric label noise is generated by flipping every label to other labels with uniform flip rates $\{ 20\%, 40\%\}$. The asymmetric label noise by flipping label $i$ to label $i+1$ (except that label 9 is flipped to label 0) with pair-wise flip rates $ \{20\%, 40\% \}$. We employed the code of \citet{han2018co} for generating noisy labels for CIFAR-10 and CIFAR-100.

\textbf{Hyperparameter Settings:} The learning rate setting: $\eta=1$ for PNM; $\eta=0.1$ for SGD (with Momentum). The batch size is set to $128$. We use the common setting $\lambda=0.0001$ in weight decay for both PNM and SGD. For the learning rate schedule, the learning rate is divided by 10 at the epoch of $\{80,160\}$. PNM uses $\beta_{0}=70$ for asymmetric label noise and $\beta_{0}=80$ for symmetric label noise.

\section{Supplementary Experiments}
\label{sec:ablation}

 \begin{table*}[t]
\caption{Test performance comparison of optimizers with $\lambda= 0.0001$.}
\label{table:cifarwd}
\begin{center}
\begin{small}
\begin{sc}
\resizebox{\textwidth}{!}{%
\begin{tabular}{ll|llllllllll}
\toprule
Dataset & Model    &  PNM & AdaPNM &SGD &Adam  &AMSGrad & AdamW & AdaBound & Padam & Yogi & RAdam  \\
\midrule
CIFAR-10   &ResNet18         & $\mathbf{4.86}$ &  $5.11$  & $5.58$  & $6.08 $ & $5.72$ &  $5.33$ &  $6.87$ & $5.83$ & $5.43$ & $5.81$ \\
                   &VGG16            &  $6.58$ & $\mathbf{6.43}$  & $6.92$  & $7.04$  & $6.68$  & $6.45$ & $7.33$ & $6.74$ & $6.69$ & $6.73$ \\
CIFAR-100 &ResNet34        &    $26.18$ & $\mathbf{22.94}$  & $24.92$  & $25.56$  & $24.74$ &  $23.61$ & $25.67$ & $25.39$ & $23.72$ & $25.65$ \\
                   &DenseNet121   &   $\mathbf{20.68}$ & $21.44$  & $20.98$   & $24.39 $  & $22.80$ &  $22.23$ & $24.23$ & $22.26$ & $22.40$ & $22.40$ \\
                   &GoogLeNet      &   $21.91$ & $\mathbf{21.62}$  & $21.89$ & $24.60$  & $24.05$ &  $21.71$ & $25.03$ & $ 26.69 $ & $22.56$ & $22.35$ \\
\bottomrule
\end{tabular}
}
\end{sc}
\end{small}
\end{center}
\end{table*}

\begin{figure}
\centering
\includegraphics[width=0.4\textwidth]{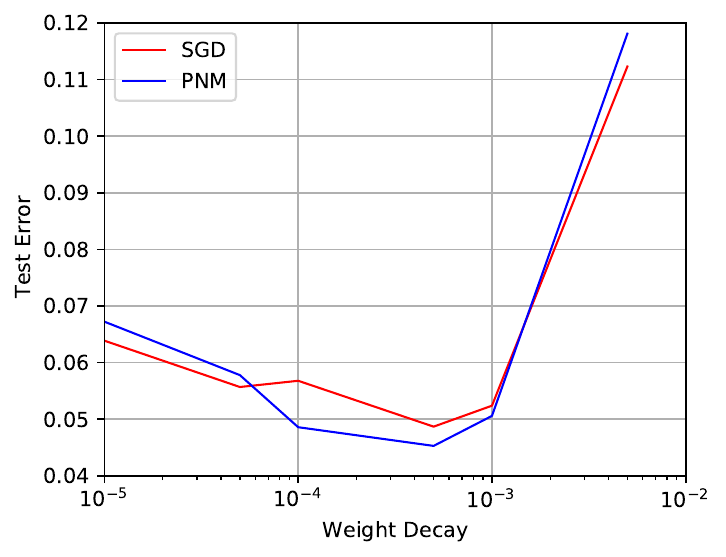} 
\caption{ We compare the generalization of PNM and SGD under various weight decay hyperparameters by training ResNet18 on CIFAR-10. The optimal performance of PNM is better than the conventional SGD.}
 \label{fig:wdresnetpnm}    
\end{figure}

\textbf{On the strength of weight decay.} We display the experimental results with $\lambda=0.0001$ in Table \ref{table:cifarwd}. Popular optimizers with $\lambda = 0.0005$ can yield test results than $\lambda = 0.0001$ on CIFAR-10 and CIFAR-100. Some related papers used $\lambda = 0.0001$ directly and obtained lower baseline performance than ours. Figure \ref{fig:wdresnetpnm} also supports that PNM outperforms Momentum under a wide range of weight decay.

\textbf{On the type of weight decay.} We have two observations in Figure \ref{fig:cifarpnml2}. First, PNM favors decoupled weight decay over $L_{2}$ regularization. Second, with either $L_{2}$ regularization or decoupled weight decay, PNM generalizes significantly better than SGD.

\textbf{On learning rate schedulers.} Figure \ref{fig:cifarcosine}, with cosine annealing and warm restart schedulers, PNM and AdaPNM also yields better results than SGD.

\textbf{Learning with noisy labels.} We also run PNM and SGD on CIFAR-10 with $20\%$ label noise for comparing the robustness to noise memorization. Figure \ref{fig:labelnoisepnmall} shows that PNM consistently outperforms SGD.

\begin{figure}
\center
\subfigure[ResNet18 on CIFAR-10]{\includegraphics[width =0.32\columnwidth ]{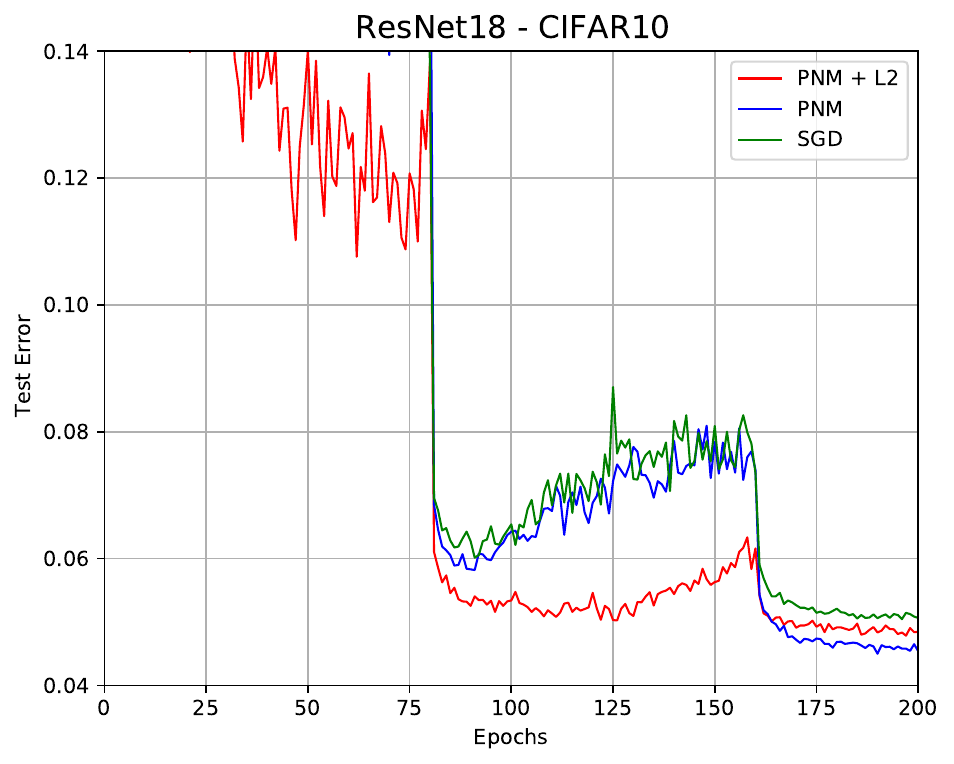}} 
\subfigure[ResNet34 on CIFAR-100]{\includegraphics[width =0.32\columnwidth ]{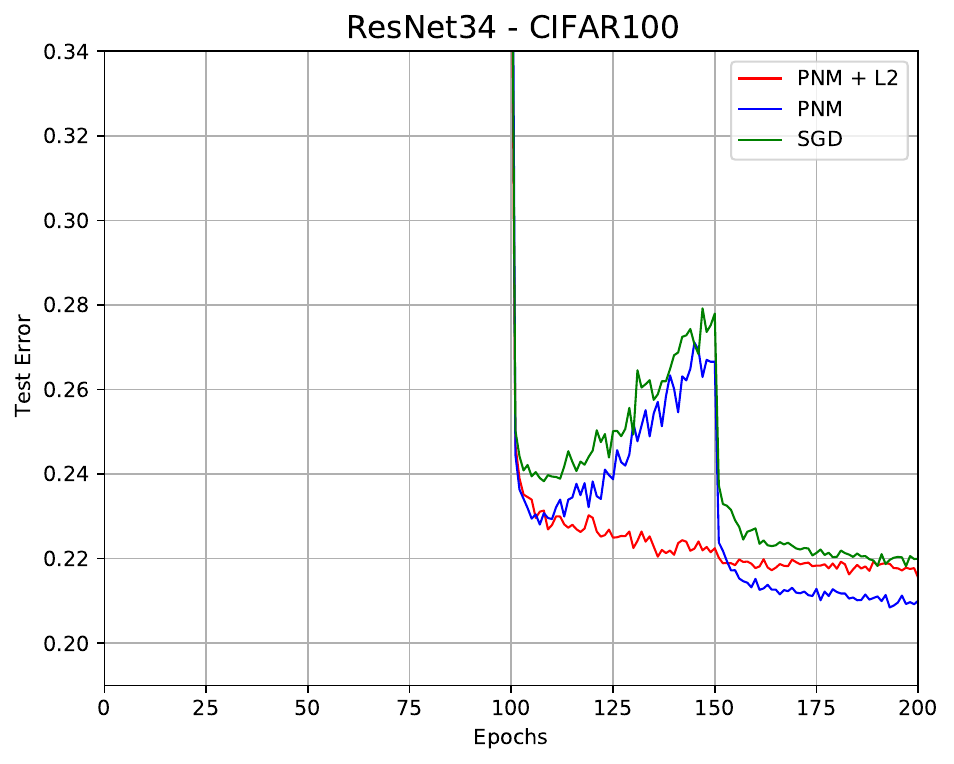}} 
\subfigure[GoogLeNet on CIFAR-100]{\includegraphics[width =0.32\columnwidth ]{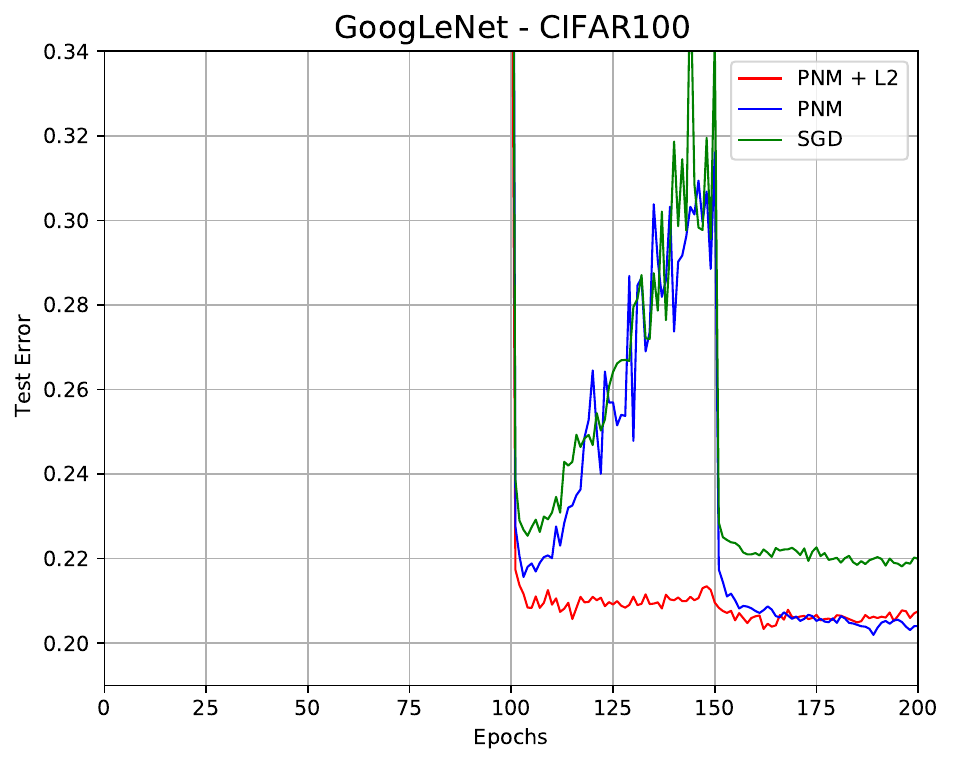}}  \\ 
\subfigure{\includegraphics[width =0.32\columnwidth ]{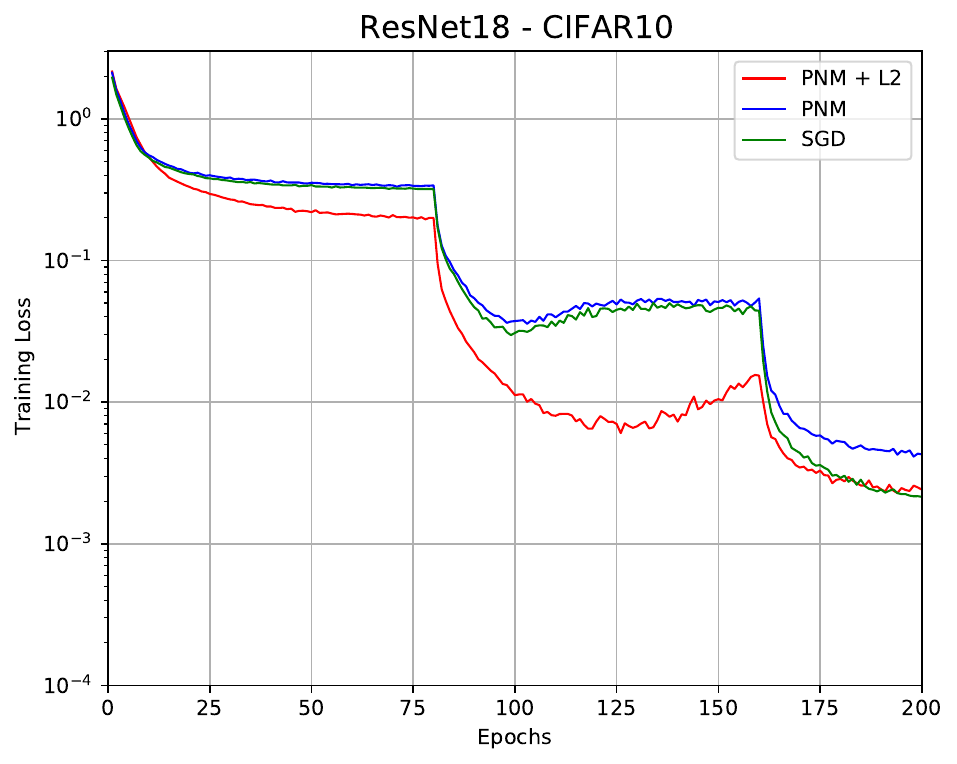}} 
\subfigure{\includegraphics[width =0.32\columnwidth ]{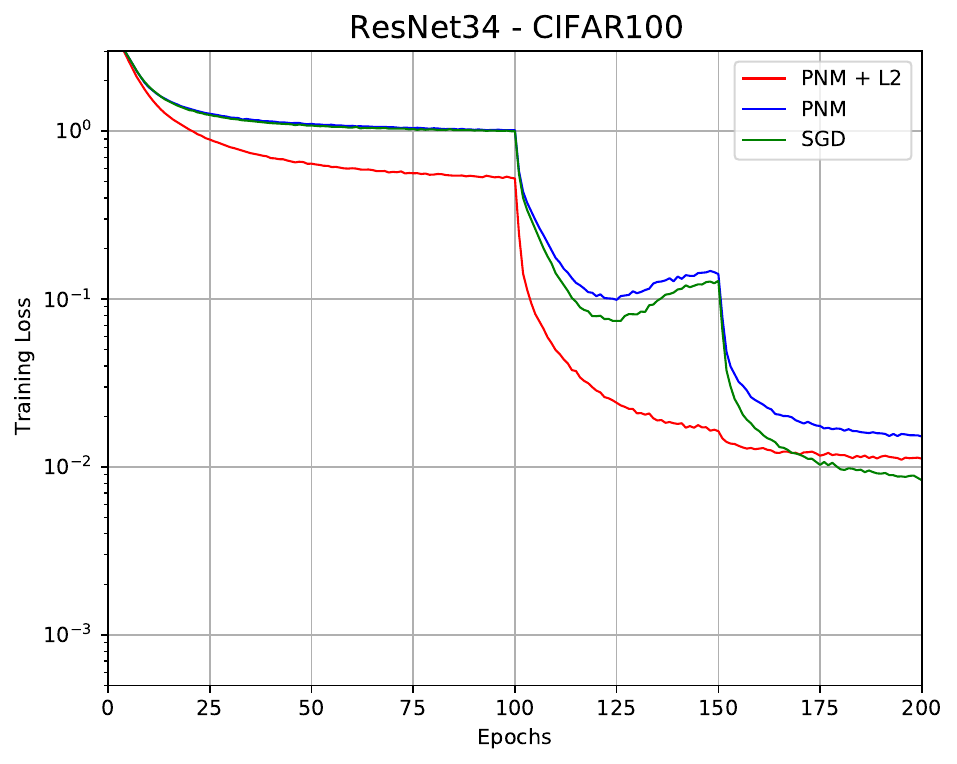}}
\subfigure{\includegraphics[width =0.32\columnwidth ]{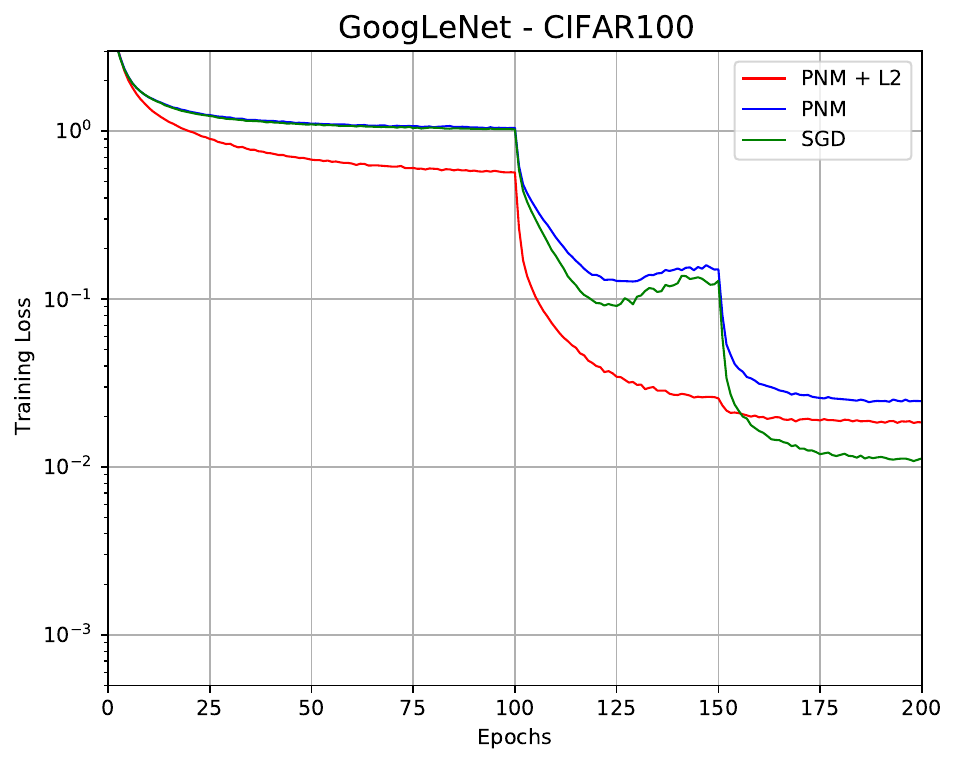}}   \\
\caption{ The learning curves of ResNet18, ResNet34, and GoogLeNet  on CIFAR-10 and CIFAR-100, respectively. Top Row: Test curves. Bottom Row: Training curves. We have two observations. First, PNM favors decoupled weight decay over $L_{2}$.  Second, with either $L_{2}$ regularization or decoupled weight decay, PNM generalizes significantly better than SGD.}
 \label{fig:cifarpnml2}
\end{figure}

\begin{figure}
\center
\subfigure[ResNet18 on CIFAR-10]{\includegraphics[width =0.32\columnwidth ]{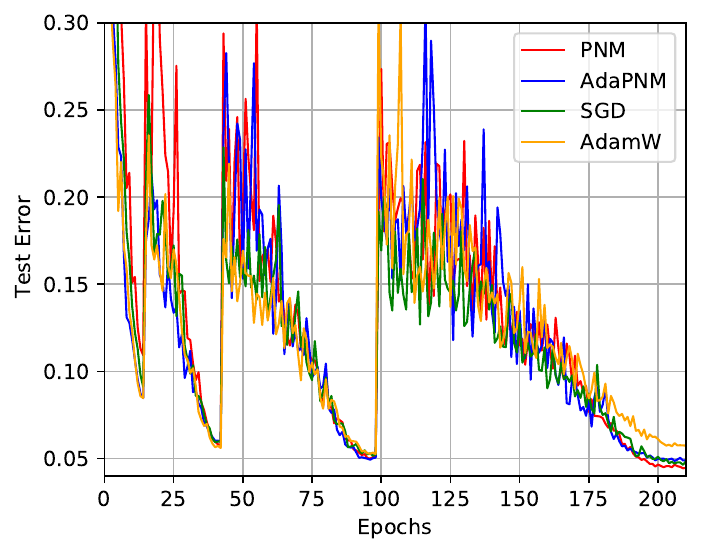}} 
\subfigure[VGG16 on CIFAR-10]{\includegraphics[width =0.32\columnwidth ]{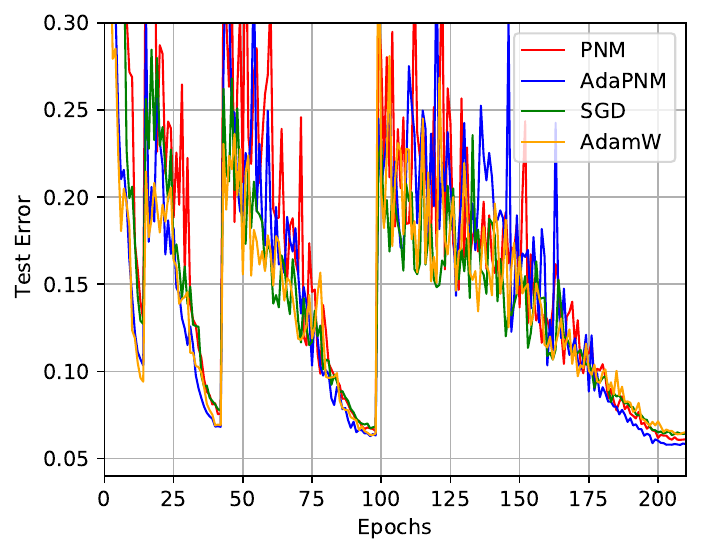}}  
\subfigure[ResNet34 on CIFAR-100]{\includegraphics[width =0.32\columnwidth ]{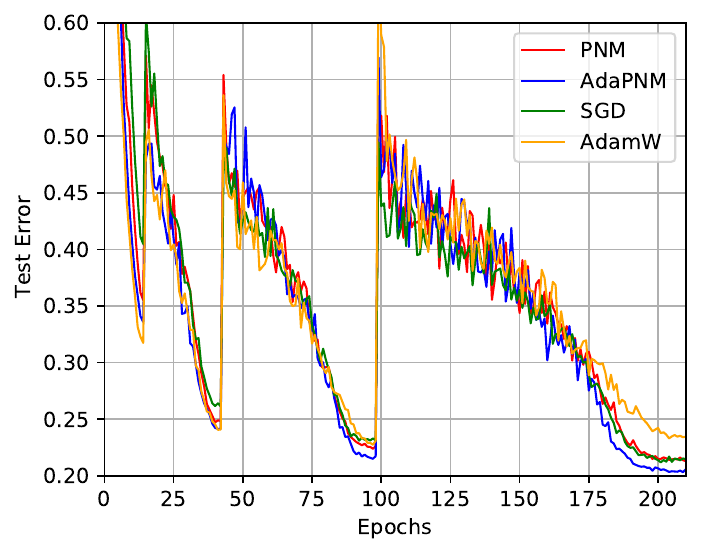}} \\
\caption{ The learning curves of ResNet18, VGG16, and ResNet34 on CIFAR-10 and CIFAR-100 with cosine annealing and warm restart schedulers. PNM and AdaPNM yields better results than SGD and AdamW.}
 \label{fig:cifarcosine}
\end{figure}

\begin{figure}
\centering
\subfigure{\includegraphics[width=0.32\columnwidth]{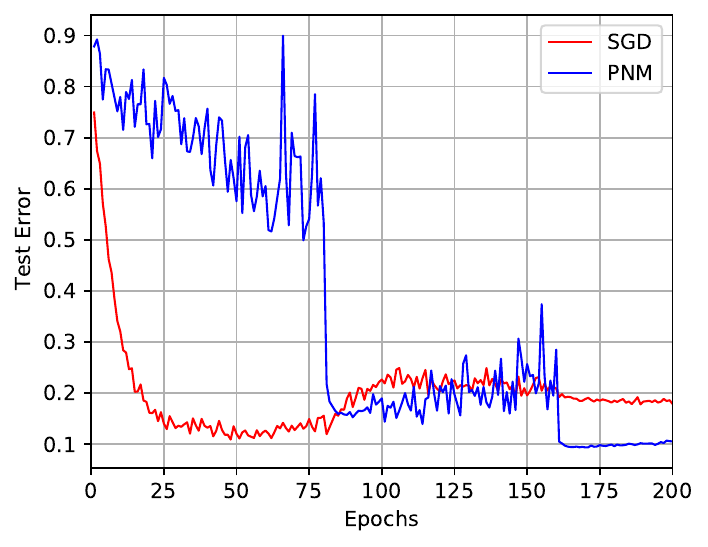}} 
\subfigure{\includegraphics[width=0.32\columnwidth]{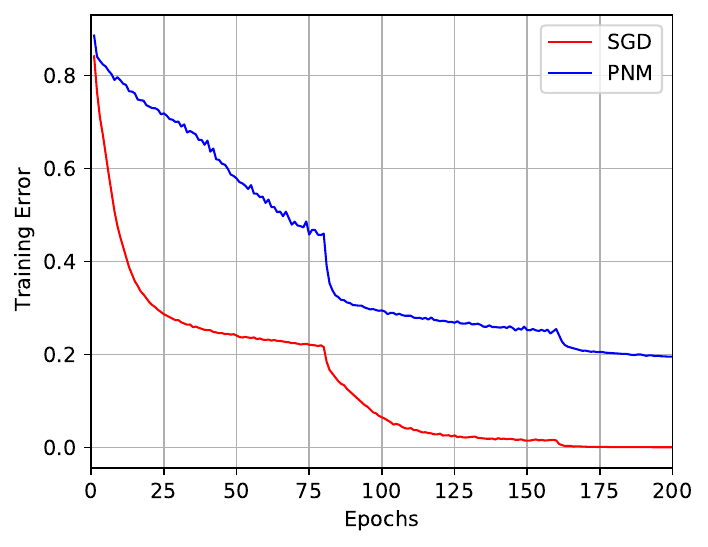}} \\
\subfigure{\includegraphics[width=0.32\columnwidth]{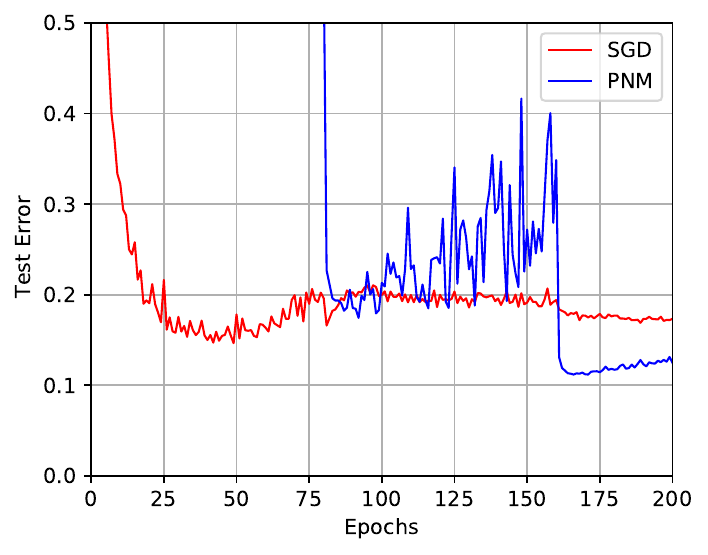}} 
\subfigure{\includegraphics[width=0.32\columnwidth]{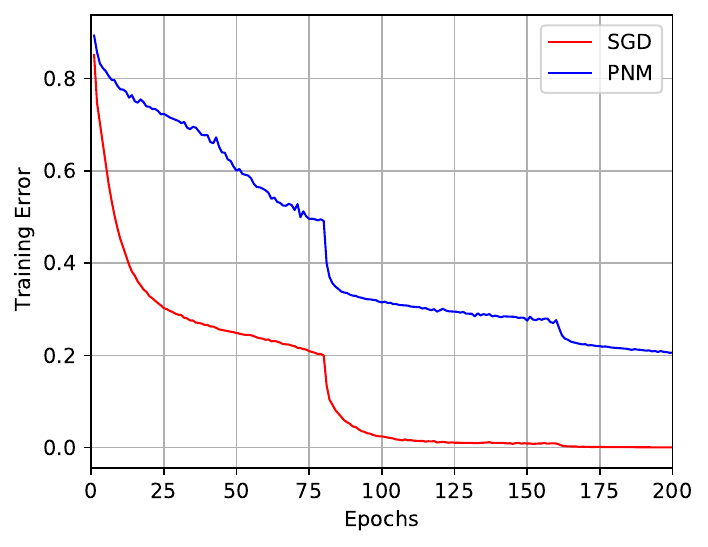}}
\caption{We compare PNM and SGD (with Momentum) by training ResNet34 on CIFAR-10 with $20\%$ asymmetric label noise and $20\%$ symmetric label noise. Left: Test Curve. Right: Training Curve. We observe that PNM with a large $\beta_{0}=70$ may effectively relieve memorizing noisy labels and almost only learn clean labels, while SGD almost memorize all noisy labels and has a nearly zero training error.}
 \label{fig:labelnoisepnmall}
\end{figure}

\begin{algorithm}
 \caption{Adaptive Positive-Negative Momentum (standard)} 
 \label{algo:adapnmstandard}
      $ m_{t} = \beta_{1}^{2} m_{t-2} + (1-\beta_{1}^{2}) g_{t} $\; \\
      $ \hat{m}_{t} =\frac{ (1 +  \beta_{0} ) m_{t}  - \beta_{0} m_{t-1} }{(1-\beta_{1}^{t})\sqrt{(1 +  \beta_{0} )^{2} + \beta_{0}^{2}}}   $\; \\
      $ v_{t} = \beta_{2}  v_{t-1} + (1 - \beta_{2} ) g_{t}^{2} $\; \\
      $ \hat{v}_{t} = \frac{v_{t}}{1-\beta_{2}^{t}} $\;  \\
      $ \theta_{t+1} = \theta_{t} -  \frac{\eta}{\sqrt{\hat{v}_{t}} + \epsilon} \hat{m}_{t} $\; 
\end{algorithm}

\section{Stochastic Gradient Noise Analysis}
\label{sec:noiseanalysis}

\begin{figure}
 \center
\subfigure[Pretrained Models]{\includegraphics[width =0.24\columnwidth ]{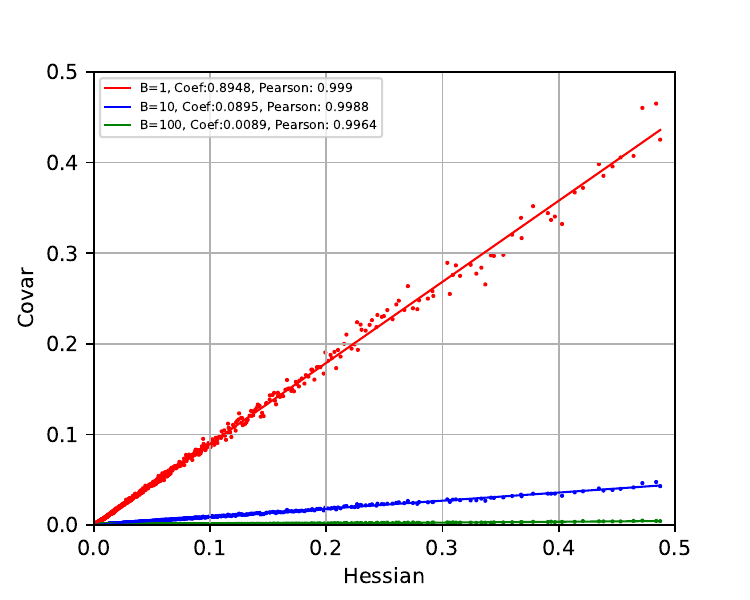}} 
\subfigure[Pretrained Models]{\includegraphics[width =0.24\columnwidth ]{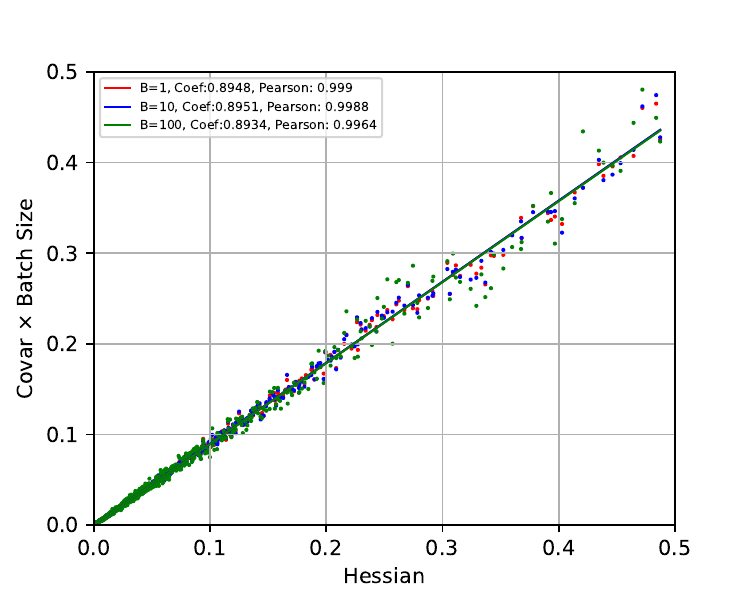}}  
\subfigure[Random Models]{\includegraphics[width =0.24\columnwidth ]{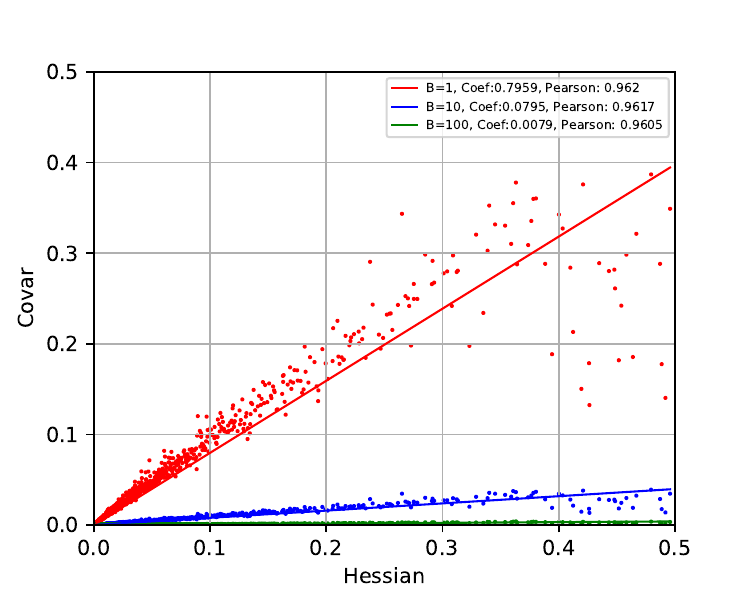}} 
\subfigure[Random Models]{\includegraphics[width =0.24\columnwidth ]{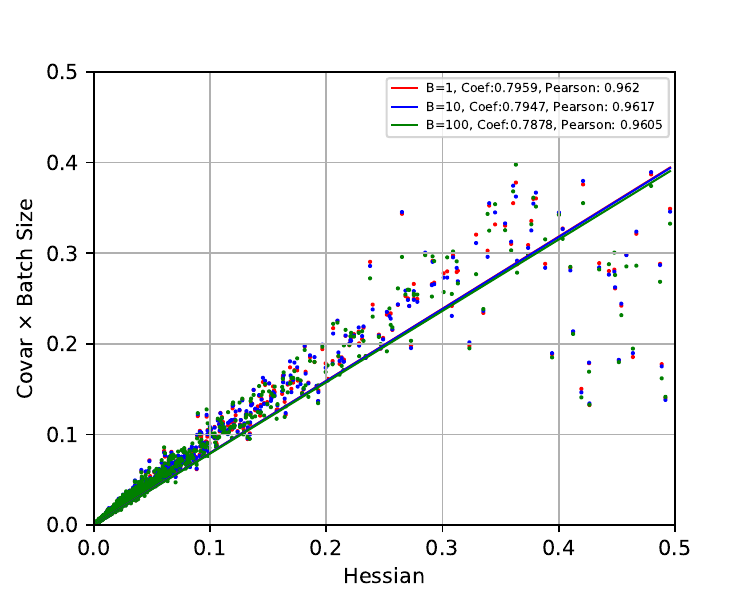}}  
\caption{\citet{xie2020diffusion} verified Equation \eqref{eq:covarhessian} by training pretrained and random three-layer fully-connected networks on MNIST \citep{lecun1998mnist}. Stochastic gradient noise covariance is still approximately proportional to the Hessian and inverse to the batch size $B$ even not around critical points. \citep{xie2020diffusion}}
\label{fig:DCH}
\end{figure}

In Figure \ref{fig:noisetype} and Figure \ref{fig:DCH}, \citet{xie2020diffusion,xie2020adai} discussed the covariance of SGN and why SGN is approximately Gaussian in common settings.

\begin{figure}
\centering
\subfigure[\small{Gradient Noise of one minibatch across parameters}]{\includegraphics[width =0.35\columnwidth ]{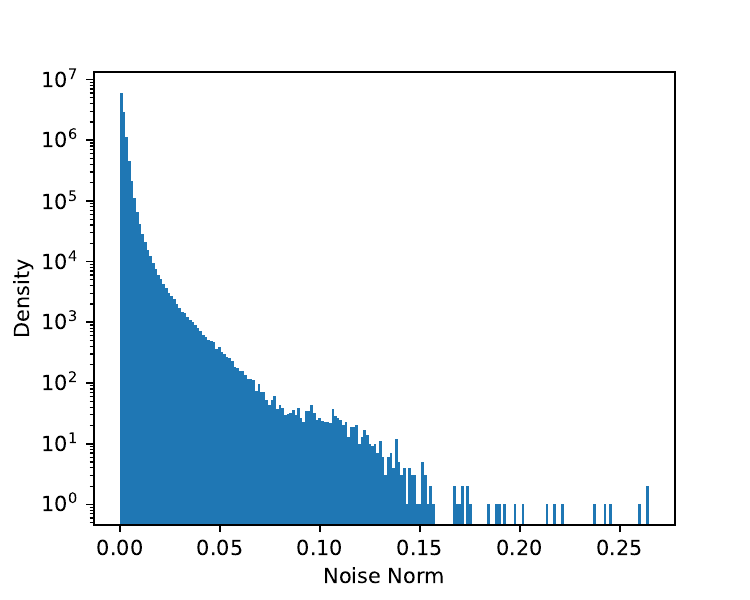}} 
\subfigure[\small{L\'evy noise}]{\includegraphics[width =0.35\columnwidth ]{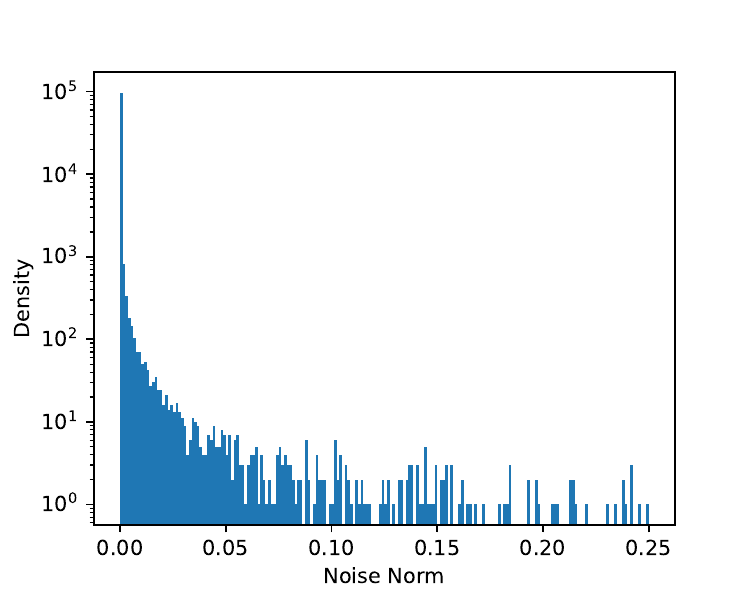}}  \\
\subfigure[\small{Gradient Noise of one parameter across minibatches}]{\includegraphics[width =0.35\columnwidth ]{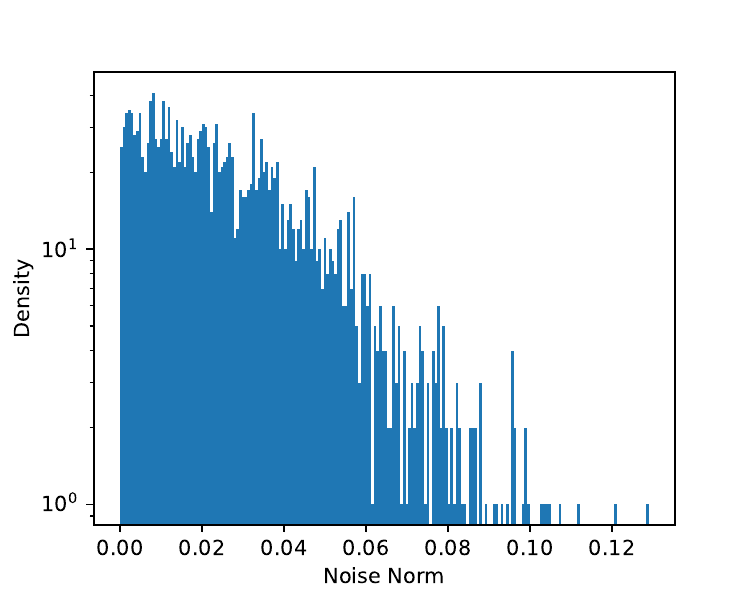}} 
\subfigure[\small{Gaussian noise}]{\includegraphics[width =0.35\columnwidth ]{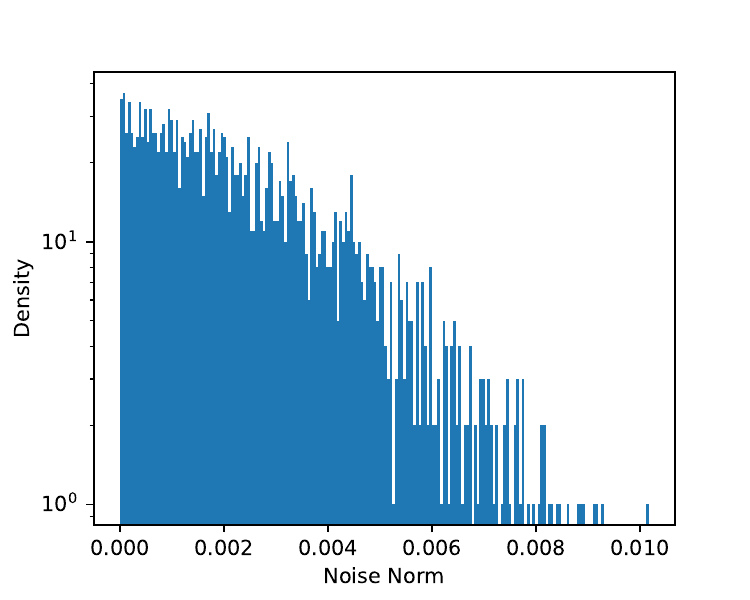}}  
\caption{The Stochastic Gradient Noise Analysis \citep{xie2020diffusion}. The histogram of the norm of the gradient noises computed with ResNet18 on CIFAR-10. Subfigure (a) follows \citet{simsekli2019tail} and computes ``stochastic gradient noise'' across parameters. Subfigure (c) follows the usual definition and computes stochastic gradient noise across minibatches. Obviously, SGN computed over minibatches is more like Gaussian noise rather than L\'evy noise.}
\label{fig:noisetype}
\end{figure}

\end{document}